\documentclass[11pt]{article}
\usepackage{times}
\usepackage[left=1.25in, top=1in, bottom=1in, right=1.25in]{geometry}

\usepackage{microtype}
\usepackage{graphicx}
\usepackage{booktabs} 

\usepackage{amsmath}
\usepackage{amssymb}
\usepackage{mathtools}
\usepackage{amsthm}
\usepackage{soul}
\usepackage{thmtools, thm-restate}

\makeatletter
\newlength\aftertitskip     \newlength\beforetitskip
\newlength\interauthorskip  \newlength\aftermaketitskip

\def\maketitle{\par
 \begingroup
   \def\thefootnote{\fnsymbol{footnote}}
   \def\@makefnmark{\hbox to 4pt{$^{\@thefnmark}$\hss}}
   \@maketitle \@thanks
 \endgroup
\setcounter{footnote}{0}
 \let\maketitle\relax \let\@maketitle\relax
 \gdef\@thanks{}\gdef\@author{}\gdef\@title{}\let\thanks\relax}

\def\@startauthor{\noindent \normalsize\bf}
\def\@endauthor{}
\def\@starteditor{\noindent \small {\bf Editor:~}}
\def\@endeditor{\normalsize}
\def\@maketitle{\vbox{\hsize\textwidth
 \linewidth\hsize \vskip \beforetitskip
 {\begin{center} \LARGE\@title \par \end{center}} \vskip \aftertitskip
 {\def\and{\unskip\enspace{\rm and}\enspace}%
  \def\addr{\small\it}%
  \def\email{\hfill\small\tt}%
  \def\name{\normalsize\bf}%
  \def\AND{\@endauthor\rm\hss \vskip \interauthorskip \@startauthor}
  \@startauthor \@author \@endauthor}
}}

\makeatother

\title{Robust Training of Neural Networks using \\ Scale Invariant Architectures}

\author{\name Zhiyuan Li\thanks{Work done at Google Research New York} \email{zhiyuanli@cs.princeton.edu}\\
  \addr{Princeton University}\\ 
  \name Srinadh Bhojanapalli \email{bsrinadh@google.com}\\
    \addr{Google Research New York}\\
  \name Manzil Zaheer \email{manzilzaheer@google.com}\\
    \addr{Google DeepMind New York}\\
  \name Sashank J.\ Reddi \email{sashank@google.com}\\
    \addr{Google Research New York}\\
  \name Sanjiv Kumar \email{sanjivk@google.com}\\
  \addr{Google Research New York}
}

\usepackage[colorlinks,
            linkcolor=red,
            anchorcolor=blue,
            citecolor=blue
            ]{hyperref}  
\usepackage[capitalize,noabbrev]{cleveref}

\theoremstyle{plain}
\newtheorem{theorem}{Theorem}[section]

\newtheorem{lemma}[theorem]{Lemma}
\newtheorem{example}[theorem]{Example}

\newtheorem{claim}{Claim}[theorem]
\theoremstyle{definition}
\newtheorem{definition}[theorem]{Definition}
\newtheorem{assumption}[theorem]{Assumption}

\newtheorem{condition}[theorem]{Condition}

\theoremstyle{remark}
\newtheorem{remark}[theorem]{Remark}

\newcommand{\gd}{\textsc{GD}\xspace}
\newcommand{\gdwd}{\textsc{GD+WD}\xspace}
\newcommand{\sgd}{\textsc{SGD}\xspace}
\newcommand{\sgdwd}{\textsc{SGD+WD}\xspace}

\newcommand{\adagrad}{\textsc{Adagrad}\xspace}

\newcommand{\adam}{\textsc{Adam}\xspace}

\newcommand{\rmsprop}{\textsc{RMSprop}\xspace}

\newcommand{\lamb}{\textsc{Lamb}\xspace}
\newcommand{\bert}{\textsc{Bert}\xspace}
\newcommand{\resnet}{\textsc{ResNet}\xspace}

\newcommand{\adafactor}{\textsc{Adafactor}\xspace}

\newcommand{\sibert}{\textsc{Sibert}\xspace}

\crefname{claim}{claim}{claims}
\crefname{conjecture}{conjecture}{conjectures}
\crefname{assumption}{assumption}{assumptions}
\crefname{condition}{condition}{conditions}

\usepackage[textsize=tiny]{todonotes}

\usepackage{math_commands}
\usepackage{amsfonts}       
\usepackage{nicefrac}       
\usepackage{microtype}      
\usepackage{xcolor}         
\usepackage{bm}
\usepackage{multirow}
\usepackage{natbib}
\usepackage{makecell}
\usepackage{booktabs}
\usepackage{array}
\usepackage{algorithm}
\usepackage{algorithmic}
\usepackage{dsfont}
\usepackage{csquotes}
\usepackage{cancel}
\usepackage{enumitem}
\usepackage{thmtools} 
\usepackage{thm-restate}
\usepackage{caption}
\usepackage{subcaption}
\usepackage{makecell}
\usepackage{framed}
\usepackage{xspace}

\newtheorem{example-shaded}{Example}[section]
\definecolor{light-gray}{gray}{0.95}
\definecolor{shadecolor}{named}{light-gray}

\newcommand{\normxx}[1]{\norm{\bx(#1)}_2^2}

\newcommand{\gnormxx}[1]{\norm{\nabla L(\bx(#1))}_2^2}

\renewcommand{\algorithmicrequire}{\textbf{Input:}}
\renewcommand{\algorithmicensure}{\textbf{Output:}}

\begin{document}

\maketitle

\begin{abstract}
In contrast to \sgd, adaptive gradient methods like \adam allow robust training of modern deep networks, especially large language models. 
However, the use of adaptivity not only comes at the cost of extra memory but also raises the fundamental question: \emph{can non-adaptive methods like \sgd enjoy similar benefits}? 
In this paper, we provide an affirmative answer to this question by proposing to achieve both robust and memory-efficient training via the following general recipe: 
(1) modify the architecture and make it \emph{scale invariant}, i.e. the scale of parameter doesn't affect the output of the network, 
(2) train with \sgd and weight decay, and optionally 
(3) clip the global gradient norm proportional to weight norm multiplied by $\sqrt{\tfrac{2\lambda}{\eta}}$, where $\eta$ is learning rate and $\lambda$ is weight decay. 
We show that this general approach is robust to rescaling of parameter and loss by proving that its convergence only depends logarithmically on the scale of initialization and loss, whereas the standard \sgd might not even converge for many initializations. 
Following our recipe, we design a scale invariant version of \bert, called \sibert, which when trained simply by vanilla \sgd achieves performance comparable to \bert trained by adaptive methods like \adam on downstream tasks.
\end{abstract}

\section{Introduction}

Neural architectures like transformers are the cornerstone for modern machine learning applications. 
However, training them is difficult and often results in training instability \cite{Liu20, Zhang20}.
To enable stable training, one typically requires adaptive and carefully tuned learning rates. However, the reason behind this issue is not very well-understood and lacks a formal treatment. 

In this paper, we hypothesize that a primary cause of such behavior is the $k$-homogeneous ($k\ge 2$) nature of the network i.e., property where network's output is scaled by $s^k$ when its parameters are scaled by $s$. To illustrate our point, we consider the following instructive toy model.

\begin{shaded}
\begin{example}
Consider logistic regression with $1$-dimensional non-separable data, $\{z_i,y_i\}_{i=1}^n\in \left(\RR\times \{\pm1\}\right)^n$. The loss is defined as $L(x_1,,\ldots, x_{2k}) =\tilde{L}(X) :=- \sum_{i=1}^n\ln (1+e^{-z_iy_iX})$ where $X = x_1\ldots x_{2k}$ and $k\ge 2$.

Since $\tilde{L}$ is convex with bounded smoothness in $X$, there exists step size that are independent of any initialization that allow GD to converge to the optimal solution.
In sharp contrast, the reparametrized loss $L(x_1,,\ldots, x_{2k})$ with $2k$-homogeneous structure does not enjoy this nice stability property --- the learning rate has to be tuned according to the initialization. In particular, when $\eta \ge \frac{2}{|\nabla \tilde{L}(X(0))|}(X(0))^{\frac{1}{k}-1}$ and $X(0)>X^*$ where $X^*>0$ is the global minimizer, $X(t)$ will monotonically increase and explode, if all $x_i$ are initialized to be the same.
\end{example}
\end{shaded}

We refer the reader to Appendix \ref{sec:intro_analysis} for a formal justification of this example. In the above example, the success of optimization is very sensitive to the right choice of the learning rate that depends on the initialization. Furthermore, the training cannot recover once the norm explodes due to large gradient update. 

In the above one-dimensional example it is still possible to find a small workable learning rate by extensive grid search that depends on the initial point, however, the situation can get worse when the $k$-homogeneous structure has an unbalanced initialization as below.
\begin{shaded}
\begin{example}
Consider solving low-rank matrix decomposition by Gradient Descent. Let $L(A,B) = \frac{1}{2}\norm{AB^\top-Y}_2^2$ where $A,B\in\RR^{d\times r}$ are both initialized i.i.d. gaussian with covariance $\sigma_A^2\gg\sigma_B^2\approx \sigma_A^{-2}$, $Y\in\RR^{d\times d}$ and $d\gg r$.  

Solving this optimization problem requires  $A$ and $B$ learning the column and row space of $Y$ respectively, but the unbalanced initialization will force the learning rate to be small enough such that $B$ does not explode and, thus, $A$ is almost frozen. To see this, note in the standard convergence analysis of GD, we need LR smaller than $2/\norm{\nabla^2 L}$ to ensure the Descent Lemma holds, \emph{i.e.}, loss decreases in a single step. Here we have that the smoothness w.r.t $A$ (fixing $B$) is $\lambda_{max}(BB^T)$ and the smoothness w.r.t. $B$ (fixing $A$) is $\lambda_{max}(AA^T)$. Thus, LR can be at most $O(\frac{1}{\sigma_A^2})$, but the gradient of $A$ is only of magnitude $O(\sigma_B)$, resulting in $A$  learning the column space slowly. Specifically, when $d=1$ and $Y=0$ and for any $r\ge 1$,  choosing $\eta> \frac{4}{\norm{\nabla_B^2 L}}$ will cause \gd to provably explode~\citep{lewkowycz2020large}.
\end{example}
\end{shaded}

Similar issues can exist in deep neural networks as the $k$-homogeneous structure is quite common. For instance, \cite{Liu20} identified the gradient norm varies with depth and that no single learning rate is globally optimal for all layers.
To this end, one has to resort to adaptive methods like \adam to handle the $k$-homogeneous structure of deep networks and allow for its robust training.
However, this not only comes at the expense of higher memory, but also raises the key question of our interest: 

\emph{Can non-adaptive methods like SGD enjoy fast and robust convergence without training instability?} 

Answering this question, requires us to first define our notion of robustness. In this paper, we primarily aim for three aspects of robustness by preventing: explosion of parameters (e.g. due to frequent large gradient updates), slow progress in training (e.g. due to loss plateaus) and loss explosion or spikes (e.g. due to possibly infrequent large magnitude updates). In this paper, we propose a 
simple yet powerful general approach for achieving such fast and robust convergence. At a high level, our recipe for robust training includes three key ingredients: 
\begin{enumerate}
\item \emph{Designing architectural scale invariance which allows for improved training stability and prevents explosion of the parameters}. We show that by using scale invariance in the architecture (i.e., making the network $0$-homogeneous), one can effectively control the gradient updates when the parameter norm is large.

\item \emph{Using SGD with weight decay for training, wherein enabling weight decay improves training efficiency under rescaling of loss and initialization}. While scale invariance prevents explosion of parameters, the training convergence has strong dependence on initialization scale and learning rate, which  can make training inefficient in face of  parameter and initialization rescaling. Use of SGD with weight decay circumvents this issue. 
\item 
\emph{Using a novel Relative Global Clipping to prevent spikes in training loss and improve overall convergence speed}. Although scale invariance in the architecture already guarantees the training stability, it does not prevent severe non-monotonic loss explosion. By using a new global clipping approach, we show that one can prevent such loss explosions effectively.
\end{enumerate}

We show that this surprisingly simple training recipe can not only improve the memory efficiency over adaptive methods but also achieves robust training. In light of the above background, we list our main contributions below.
\begin{itemize}
	\item In \Cref{sec:methods}, we propose a new general recipe for memory efficient, robust training using (1) scale invariant architecture; (2) SGD+WD for training and (3) a novel clipping rule, called Relative Global Clipping, for clipping the updates. Following this recipe, we design a new variant of \bert called Scale Invariant \bert (\sibert).
	\item In \Cref{sec:gd,sec:main_sgd_wd}, we prove the convergence rate to the approximate first order point for GD and \sgd for scale invariant loss. 
	We show that \sgdwd matches the standard rates, even without the knowledge about the smoothness of loss and is robust to the scale of initialization or loss.
	\item In \Cref{sec:main_clip_sgd}, we show \sgdwd with Relative Global Clipping has better parameter norm convergence via a novel analysis. With assumptions that the clipping does not bring too much bias in expected gradients, we show similar convergence result to \sgdwd.
	\item In our empirical analysis in \Cref{sec:exp}, we demonstrate that \sibert trained using simple \sgd can achieve performance comparable to standard \bert trained with \adam. Furthermore, we also verify our theoretical claims. To our knowledge, this is the first time a \bert-like model has been effectively trained using vanilla \sgd.
\end{itemize}

\section{Related Work \& Background}

The literature on adaptive methods and scale invariance in neural networks is vast, so we only discuss works that are most relevant to our paper.

\paragraph{Adaptive Methods \& Clipping Methods.} Adaptive learning rates
have long been studied \cite{polyak1987introduction}. In machine learning, adaptive learning rates have been popularized by \adagrad, which particularly benefits from sparse stochastic
gradients \cite{duchi2011adaptive}. Inspired by \adagrad, several adaptive methods, like \adam, \rmsprop and its variants have been proposed in the deep learning community \cite{kingma2014adam,tieleman2012lecture,reddi2018convergence,You2020lamb,shazeer2018adafactor}. These approaches have been crucial in the success of many deep learning applications \cite{vaswani2017attention,devlin2018bert,raffel2019exploring}. Several works have studied the benefits of adaptive methods in deep learning settings (e.g. \cite{Liu20,Zhang20}).  However, as mentioned earlier, these benefits come at the cost of  computational and memory efficiency.   \citet{anil2019memory} proposed a variant of  \adagrad requiring fewer parameters for adaptivity, but still requires momentum. \adafactor~\citep{shazeer2018adafactor} removes momentum and uses much fewer adaptivity parameters, but for large models, \adafactor still needs momentum to ensure training stability~\citep{chowdhery2022palm}.  Our approach is also related to normalized and projected gradient descent, which has been studied for quasi-convex and non-convex settings (e.g. see \cite{hazan2015beyond, levy2016power, Huang2017ProjectionBW}). However, these methods have seen very limited success.

Clipping based optimization methods, especially gradient clipping, are widely used in deep learning applications to improve training stability or ensure privacy \cite{pascanu13,Chen20,ZhangHSJ20}. These approaches typically use a constant threshold to clip the gradients before the update. However, choosing this threshold is difficult and requires careful tuning. Adaptive variants of clipping methods partially alleviate this issue and are closely related to adaptive methods \cite{Zhang20}; however, they again incur additional computation and memory costs. 

\paragraph{Scale Invariance in deep networks.} Various normalization schemes are the main source of scale invariance in deep learning, \emph{e.g.}, BatchNorm~\cite{ioffe2015batch}, LayerNorm~\cite{ba2016layer}, Weight Normalization~\cite{salimans2016weight}, GroupNorm~\cite{wu2018group}, InstanceNorm~\cite{ulyanov2016instance}.   Scale invariance from normalization  allows \gd and \sgd to converge to stationary points from any initialization and with any learning rate, in $O(T^{-1/2})$ and $\tilde{O}(T^{-1/4})$ rates respectively \cite{arora2018theoretical}. The interplay between \sgd, scale invariance and WD has also been well studied. It was shown that the effect of WD for normalized networks can be replaced by LR schedules \cite{hoffer2018norm, zhang2018three}. \citet{li2019exponential} formally builds the equivalence between \sgdwd and \sgd with an exponential increasing LR schedule for scale invariant loss. \citet{van2017l2} first proposed the notion of effective LR, $\eta/\norm{\bx}_2^2$, for normalized networks, and showed that the unique stationary value of $\norm{\bx}_2^4$ is proportional to $\lambda/\eta$, where $\eta$ is LR and $\lambda$ is WD. \citet{li2020reconciling} proved that the parameter norm always converges to the above value by modeling \sgd as Stochastic Differential Equation.  \citet{wan2020spherical} proved the parameter norm converges to the same value directly for \sgdwd, but only in expectation.

\subsection{Preliminary}
In this section we present the definition of scale invariant functions and some of their useful properties. For $\bx\in\RR^d$, we define $\overline{\bx}: = \frac{\bx}{\norm{\bx}_2}$. We say a function is $\mathcal{C}^k$ iff it is $k$-times continuously differentiable.
\begin{definition} 
	Given a cone $U\subset \RR^d$, we say a function $f:U\to \RR$ is \emph{(positively) $k$-homogeneous} or \emph{of homogeneity of degree $k$} iff for any $c>0$ and $x\in U$, $f(c\bx)=c^kf(\bx)$. We say a function is \emph{scale invariant} iff it is $0$-homogeneous. 
\end{definition}
Now we present some useful properties of the derivatives of homogeneous functions.
\begin{theorem}[Euler's Homogeneous Function Theorem]\label{thm:euler_homo}
	For any  $k$-homogeneous $\mathcal{C}^1$ function $f$, it holds that $\inner{\nabla f(\bx)}{\bx} = kf(\bx)$.
\end{theorem}

\begin{lemma}\label{lem:homo}
For any $k$-homogeneous $\mathcal{C}^l$ function $f$, $\nabla^l f$ is $k-l$ homogeneous.	
\end{lemma}

\begin{lemma}[Equivalent Scaling]\label{lem:equivalent_scaling} The  properties below hold (and generalize to stochastic loss):
    \begin{enumerate}
        \item For any loss $L$, LR $\eta$,  WD $\lambda$ and initialization $\bx(0)$, rescaling $(L,\eta,\lambda,\bx(0))\to (cL,\eta/c,c\lambda,\bx(0))$ doesn't change GD iterate   $\bx(t)$ for any $t\ge0$.
        \item For any scale invariant loss $L$, LR $\eta$,  WD $\lambda$ and initialization $\bx(0)$, rescaling $(L,\eta,\lambda,\bx(0))\to (L,c^2\eta,\lambda/c^2,c\bx(0))$ doesn't change the direction of GD iterate $\overline \bx(t)$ for any $t\ge0$. (see Lemma 2.4 in \cite{li2019exponential})
    \end{enumerate}
\end{lemma}

\section{Methods}\label{sec:methods}
In this section, we provide a more detailed description of our recipe for robust and memory-efficient network training, which includes three building blocks: (1) scale invariant architecture (Section~\ref{sec:si-design}), (2) SGD with Weight Decay (Section~\ref{sec:sgd-wd-algorithm}) and optionally (3) the \emph{Relative Global Clipping} (Section~\ref{sec:rel-clipping} and \Cref{alg:clipped_sgd}). 
\begin{algorithm}
\caption{ $\sqrt{C}$-Clipped SGD + WD}\label{alg:clipped_sgd}
\begin{algorithmic}
\REQUIRE Total steps $T$, Scale invariant loss $\{L_t\}_{t\ge 1}^T$,  initialization $\bx(0)$, LR $\eta$, WD $\lambda$, clipping factor $C>1$ ($C=\infty\Leftrightarrow$  no clipping).\\
\FOR {$t=0$ \textbf{to} $T-1$}
\STATE $N_t\gets \min\left\{\sqrt{\frac{2C\lambda}{\eta}}\norm{\bx(t)}_2, \norm{\nabla L_{t}(\bx(t))}_2\right\}$.\\

\STATE	$\bx(t+1) \gets (1-\eta\lambda) \bx(t) - \eta  N_t \frac{\nabla L_{t}(\bx(t))}{\norm{\nabla L_{t}(\bx(t))}_2}$.
\ENDFOR
\end{algorithmic}
\end{algorithm}

\subsection{Designing Scaling Invariant Architectures}
\label{sec:si-design}

We first revisit an approach for introducing scale invariance in neural networks, which is presented in \cite{li2019exponential}. Viewing the neural network computation as a directed graph, the high level idea is to ensure same homogeneity degree of different edges reaching a node. For example in a \resnet block, the output from an affine transform is added back to the input $z$ from the previous layer yielding $z + \text{Aff}(z)$. Now if we scale all the network parameters by $c$, both $z$ and $\text{Aff}(z)$ should have the same degree of homogeneity and scale as $c^k$. Otherwise the network is no longer homogeneous and, hence, cannot be scale invariant.

In this paper, we apply the above design philosophy to develop a scale invariant version of \bert~\citep{devlin2018bert} --- a transformer based model. A transformer has two main building blocks that need to be made scale invariant -- residual block and Attention~\cite{vaswani2017attention}. For residual block, \citet{li2019exponential} already demonstrated how to make both the PreNorm and PostNorm version of \resnet scale invariant (see Appendix of their paper for more details). In this paper,  we use their PreNorm variant (see \Cref{fig:SI_encoder}). Furthermore, we design a novel scale invariant version of Attention block in transformer, as described below. 






\paragraph{Scale Invariant Attention:} Recall the standard self attention block computes the following for a given input $Q,K,V \in \RR^{n \times d_{model}}$:
\begin{align*} 
\attn(Q,K,V) = \softmax(\frac{Q W^Q (KW^K)^\top}{\sqrt{d_k}}) V W^V. 	
\end{align*}
Here  $W^Q, W^K\in \RR^{d_{ model}\times d_k}$ and $W^V\in \RR^{d_{ model}\times d_v}$ are affine transformations and, hence, are all 1-homogeneous transformations. The $\softmax$ function computes row wise softmax normalization. It is easy to see that standard attention is not homogeneous as softmax is itself not homogeneous. 

We design a novel Scale Invariant Attention (SI Attention) in the following way: (also see \Cref{fig:SI_attn})
\begin{align*}
\siattn(Q,K,V) = \normalization(\relu(Q W^Q (KW^K)^\top)V W^V ,
\end{align*}
where $\normalization$ denotes the row-wise normalization by sum, \emph{i.e.}, $[\normalization(A)]_{ij} = \frac{a_{ij}}{\sum_j a_{ij}}$ and $\relu(A)$ denote the element-wise max between matrix $A$ and $0$. Notably we replace the softmax with a ReLU activation followed by normalization. Both ReLU and normalization are homogeneous operations; thus, making the overall attention score computation ($\normalization( \relu(Z QK^\top Z^\top))$) scale invariant to the concatenation of all parameters $\bx$, assuming $Q,K,V$ are already positive homogeneous  to $\bx$. Due to space constraints, the full design of Scale Invariant \bert (\sibert) is relegated to  \Cref{sec:SI_design}.

\subsection{Training Algorithm: \sgd + WD}
\label{sec:sgd-wd-algorithm}
Although scale invariance can prevent parameter divergence after a large gradient update by eliminating the positive feedback between gradient and parameter norm, it alone does not ensure \sgd trains the network in a robust and efficient way. This is because, as shown in \cite{arora2018theoretical}, the parameter norm monotonically increases when \sgd is used to optimize a scale invariant loss. As a result, once the norm becomes too large (e.g  due to large gradient in some step) the training can slow down drastically as the effective LR $\frac{\eta}{\norm{\bx_t}_2^2}$ is too small; thus, preventing effective recovery from even minor training instabilities.

To tackle this issue we propose to use Weight Decay(WD) as a way to reduce the parameter norm; thereby, allowing the network to recover from slow training induced by infrequent updates of large norm. Under mild assumptions that the expectation of squared norm of stochastic gradient does not vary too much on the unit sphere, \cite{li2020reconciling,wan2020spherical} show that the parameter norm will stabilize in $O(\frac{1}{\eta\lambda})$ steps and the learning dynamics is equivalent to one on unit sphere with effective learning rate proportional to $\Theta(\sqrt{\lambda\eta})$.

Leveraging the advantage of quick norm convergence, we show that the convergence of \sgdwd is insensitive to the following three operations:  loss rescaling (A1),   initialization rescaling (A2) and re-parametrization (A3), meaning the same convergence rate (independent of scaling $c$) can be achieved, in up to $\frac{|\log c|}{\lambda\eta}$ more steps. (See formal statement in \Cref{thm:GD_main_2,thm:sgd_main}  This property reduces the effort of hyperparameter tuning and also makes training more robust when switching between different codebases and frameworks, which is likely to have different default scaling or parametrization. Also note by scale invariance of loss $L$, (A2) is equivalent to (A3).
\vspace{-0.1cm}

\begin{enumerate}
\item[(A1).] $L\to cL$, for any $c>0$.
\item[(A2).] $\bx(0)\to c\bx(0)$, for any $c>0$.
\item[(A3).] $(L,\bx(0))\to (L',c\bx(0))$, where $L'$ is defined as $L'(\bx):= L(\frac{\bx}{c})$ for any $c>0$.
\end{enumerate}

As a comparison, previous work \cite{arora2018theoretical} showed that GD  converges to $\eps$ approximate stationary point of a scale invariant loss in $O(\frac{1}{\eps^2})$ and \sgd converges in $\tilde{O}(1/\eps^4)$ steps with any initialization. However, the constant in $O(\cdot)$ scales linearly or inversely to the above scalings ($c$ in (A1-3)). This is far from satisfying, and indeed their experiments show that either large or small LR could substantially slowdown the training progress.

\subsection{Relative Global Clipping}
\label{sec:rel-clipping}

Gradient clipping is a widely used effective strategy to stabilize neural network training. However, often the clipping threshold need to be tuned based on the optimization problem and the specific gradient distribution. Furthermore, simply using a constant threshold can severely degrade the performance \citep{Zhang20}. Thus, it is unclear how the clipping threshold needs to be set for \sgdwd on scale invariant functions such that it is insensitive to rescaling of loss and reparametrization, \emph{e.g.}, (A1-3). 

To this end, we propose a clipping strategy named \emph{Relative Global Clipping} which allows consistent and robust training behavior for \sgdwd on scale invariant loss under the aforementioned operations. In particular, we propose to set the clipping threshold as $\sqrt{\frac{2C\lambda}{\eta}}\norm{\bx}_2$, where $C\ge 1$ is a hyperparamer with default value $\sqrt{C} =2$. The high level design idea is that  (1) the clipping rule should be invariant to the scalings $(L,\eta,\lambda)\to (cL,\eta/c,c\lambda)$ and $(\bx,\eta,\lambda) \to (c\bx,c^2\eta,\lambda/c^2)$ for any $c>0$, to which \sgdwd is invariant~(see \Cref{lem:equivalent_scaling}); (2) the clipping rule should only remove the extremely large  gradients and should not trigger too often to ensure that gradient after clipping remains almost unbiased.

Intuitively, the derivation of Relative Global Clipping involves the following line of reasoning: Suppose the norm of the stochastic gradient $\norm{\nabla L_\gamma(\bx)}_2$ is constant, say $\sigma$, for all data and every parameter $\bx$ on the unit sphere. In this case, we expect our clipping strategy to not be triggered since there are no extremely high stochastic gradients. Since $L_\gamma$ is scale invariant, \Cref{thm:euler_homo} implies that $\inner{\nabla L_\gamma(\bx)}{\bx}=0$. That is, 
	\begin{align}\label{eq:motivation_clipping}
\norm{\bx(t+1)}_2^2 
=&(1-\eta\lambda)^2\norm{\bx(t)}_2^2 + \eta^2\norm{\nabla L_\gamma(\bx(t))}_2^2 \nonumber \\
 = &(1-\eta\lambda)^2\norm{\bx(t)}_2^2 + \eta^2\sigma^2/\norm{\bx(t)}_2^2.
	\end{align}
It is not difficult to show the iteration \eqref{eq:motivation_clipping} has a unique stationary point, $\norm{\bx(t)}_2^2 = \sqrt{\frac{2\eta}{\lambda (2-\eta\lambda)}}\sigma$\citep{van2017l2}. In other words, at norm equilibrium, it holds
\begin{align}
\norm{\nabla L_\gamma (\bx(t))}_2 = \frac{\sigma}{\norm{\bx(t)}_2} = &\sqrt{\frac{\lambda (2-\eta\lambda)}{\eta}} \norm{\bx(t)}_2.
\end{align}

The above calculation suggests the clipping threshold should be at least $\sqrt{\frac{2\lambda }{\eta}}\norm{\bx(t)}_2$. \footnote{We drop $-\eta\lambda$ for convenience. This doesn't lead to any practical difference as $\eta\lambda$ is typically very small, \emph{e.g.} less than$ 10^{-4}$.} Furthermore, it is not difficult to check that the clipping threshold $\sqrt{\frac{2\lambda }{\eta}}\norm{\bx(t)}_2$ is indeed invariant to the above mentioned scalings $(L,\eta,\lambda)\to (cL,\eta/c,c\lambda)$ and  $(\bx,\eta,\lambda) \to (c\bx,c^2\eta,\lambda/c^2)$. 
For each hyperparameter $C>1$, the behavior of \sgdwd is consistent for different scalings (A1-3)  and it also improves the norm convergence (reducing undesirable spikes in norm while training) for \sgdwd~(see \Cref{thm:clip_sgd_main}).
Under mild assumptions that such clipping does not introduce too much bias in gradients, we show that our recipe enables convergence to approximate stationary points. Furthermore, the rate only depends \emph{logarithmically} on the initialization and loss scale, as shown in the following section.

\section{Theoretical Analysis}\label{sec:theoretical_analysis}

In this section, we provide theoretical analysis of the convergence of \sgdwd to approximate first order stationary points for scale invariant functions. We first start with the key highlights of our theoretical analysis for \sgdwd:
\begin{enumerate} 
    \item Parameter norm converges to $\Theta((\frac{\lambda}{\eta})^{\frac{1}{4}})$ in  $T_1=\tilde{O}(\frac{1}{\eta\lambda})$ steps with high probability where  $T_1$ is a function of  loss $L$, initial norm $\norm{\bx(0)}_2$, LR $\eta$ and WD $\lambda$. Moreover, $T_1(L,\norm{\bx(0)}_2,\eta,\lambda)$ changes most by $\frac{\ln |c|}{\eta\lambda}$ for operation (A1-3).
    \item After step $T_1$, convergence to first order approximate stationary point happens and the rate only depends on $\eta\lambda$ and is unaffected by operations (A1-3).
\end{enumerate}
Properties (1) and (2) suggest our results are more robust to initialization scale (by only having logarithmic dependence on it), showing the advantage of using scale invariant functions while matching the standard convergence rates for non-convex functions. Note that the standard notion of approximate stationary point, \emph{i.e.} $\bx$ with small gradient norm of $\norm{\nabla L(\bx)}_2$ is not useful for scale invariant loss, as one can simply scale up the initialization $\bx(0)$ to infinity and the gradient norm thus scales inversely. A more reasonable notion of `stationary point' is that the direction of $\bx$,  denoted by $\overline\bx:=\frac{\bx}{\norm{\bx}_2}$, has small gradient norm, as first introduced in \cite{arora2018theoretical}. We will use this definition of approximate stationary point throughout the paper. In the section we also assume $L$ is a $C^2$ and scale invariant function and  $\rho:= \max\limits_{\norm{\bx}=1} \norm{\nabla^2L (\bx)}$.
\subsection{Convergence of GD +WD}\label{sec:gd}

We first present the convergence result in the deterministic case, \emph{i.e.}, Gradient Descent over $L(\bx)+\frac{\lambda}{2}\norm{\bx}_2^2$.
\begin{align}\label{eq:GD}
	\textrm{\gdwd:}\quad \bx(t+1) = (1- \eta\lambda)\bx(t) -\eta \nabla L(\bx(t))
\end{align}


\begin{theorem}[GD+WD]\label{thm:GD_main_2} 
For $\eta\lambda\le\frac{1}{2}$, let $\bx(t)$ be defined by GD~\eqref{eq:GD}, and $T_0 =\left\lceil \frac{1}{2\eta\lambda} \left(\left\lvert\ln \frac{\norm{\bx(0)}_2^2}{\rho\pi^2\eta}\right\rvert+ 3\right)\right\rceil$. We have
\begin{align}
    \min_{t=0, \ldots, T_0} \norm{\nabla L(\overline\bx(t))}_2^2 \le 8\pi^4\rho^2\lambda\eta.
\end{align}
\end{theorem}

This bound matches the standard $O(\frac{1}{\sqrt{T}})$ convergence rate to first order stationary point for non-convex functions. Remarkably, for a given training budget $T$, once we can set $\eta\lambda$ to be $\frac{D}{T}$ where $D$ is a constant (\emph{e.g.} 10), the convergence becomes robust to the choice the hyperparameters due to just a logarithmic dependence on them. In particular, GD+WD can work with any scaling of $L$ (which affects the smoothness on unit sphere, $\rho$), LR $\eta$ and initial norm $\norm{\bx(0)}_2$, as long as $\frac{\norm{\bx(0)}_2^2}{\rho\pi^2\eta} \in [e^{-D},e^D]$ . This is in sharp contrast to GD on standard loss as it requires knowledge about the smoothness to set the optimal LR.

 
 However, one weakness of the above result is that with a fixed $\eta\lambda$, longer training does not guarantee further convergence. The intuition is that once the iterate converge in direction and the gradient vanishes, Weight Decay will dominate the dynamics and thus the norm approaches $0$, which increases the sharpness. When the sharpness gets larger than $2/\eta$, the dynamics become unstable and results in divergence. This phenomena is first observed in \citet{li2020reconciling} and verified by \citet{lobacheva2021periodic} in practical settings. This behavior can also be viewed as a special case of Edge of Stability as described in \citet{cohen2020gradient}.

\begin{proof}[Proof Sketch of \Cref{thm:GD_main_2}] Scale invariant functions do not have bounded smoothness at $0$ making it a challenge to use standard convergence analysis. Our key insight is that for scale invariant loss function, even with a fixed LR $\eta$, \gd can tune its effective LR $\frac{\eta}{\norm{\bx(t)}_2^2}$ by changing the norm. Thus once \gd passes the area of the suitable norm, the smoothness of scale invariant loss function is upper bounded by $\frac{\rho}{r^2}$ outside the ball with radius $r$ centered at $0$.

More concretely our proof consists of 2 steps. In the first step we show that \gdwd iterates pass an area of suitable norm ($\approx\sqrt{\rho\eta}$). For large initial norm, WD could bring the norm to correct scaling in log time and then converge~(\Cref{thm:GD_main}). If the initial norm is too small and the direction is not approximately stationary, then the large gradient due to the small norm will increase the parameter norm drastically in a single step~(\Cref{lem:gd_small_init}), and again Weight Decay can bring the norm down in log steps. In the second step we show that, once the norm reaches this suitable value, the descent lemma (\Cref{lem:GD_descent}) starts to hold and the convergence analysis is standard. \end{proof}

\begin{lemma}\label{lem:GD_descent}
Let $\bx(t),\bx(t+1)$ be defined as \eqref{eq:GD}, we have
\begin{align*}
    L(\bx(t)) - L(\bx(t+1))  \ge 
    \eta\left(\frac{1}{1-\eta\lambda}- \frac{\rho\eta}{2\norm{\bx(t)}_2^2(1-\eta\lambda)^2}\right)\norm{\nabla L(\bx(t))}_2^2.
\end{align*}
When $\eta\lambda\le \frac{1}{2}$, the above can be simplified into 
\begin{align*}
    L(\bx(t)) - L(\bx(t+1)) \ge \eta\left(1- \frac{2\rho\eta}{\norm{\bx(t)}_2^2}\right)\norm{\nabla L(\bx(t))}_2^2.
\end{align*}
\end{lemma}

\begin{remark}
One might wonder why the upper bounds on loss and gradient norm do not appear in \Cref{thm:GD_main_2}. This is because we are working on a  compact domain (the unit sphere) and twice-differentiability implies those bounds implicitly. (See \Cref{lem:gradient_upper_bound,lem:loss_upper_bound})
\end{remark}


\subsection{Convergence of \sgdwd}\label{sec:main_sgd_wd}

Below we present our convergence analysis for \sgdwd.

\paragraph{Setting:} Let $\Gamma$ be an index set and  $L_{\gamma}:\RR^d/\{\bm{0}\}\to \RR$ be a scale invariant loss function for each $\gamma\in \Gamma$. We denote $\E_{\gamma} L_\gamma$ by $L$. We assume the largest possible stochastic gradient norm is finite, \emph{i.e.}, $M:=\sup_{\gamma\in\Gamma}\max\limits_{\norm{\bx}=1} \norm{\nabla L_\gamma (\bx)}$. SGD is defined as \eqref{eq:SGD}.
\begin{equation}\label{eq:SGD}
\textrm{\sgdwd:}\ 	\bx(t+1) = (1- \eta\lambda)\bx(t) - \eta\nabla L_{\gamma_t}(\bx(t)),
\end{equation}
where $\gamma_t\in \Gamma$ are i.i.d. random variables. We further assume there exists constants $\undersigma$ and $\oversigma$, such that $\undersigma^2 \le \E \noise{ \bx}\le \oversigma^2$, for any $\norm{\bx}_2=1$. We finally need the following condition on $\eta \lambda$ to bound convergence.
\begin{condition}\label{assump:grad_norm_max}
 $\frac{\undersigma^2}{M^2} \ge 3e^{4\eta\lambda}\sqrt{{\lambda\eta}\ln\frac{2T^2}{\delta}}$.
\end{condition}

The \Cref{assump:grad_norm_max} is useful for proving norm convergence in high probability. In practice, typically $\eta\lambda$ is very small.  Our experiments use $\eta=0.0008$ and $\lambda=0.01$. Hence $e^{4\eta \lambda }\approx 1$, and \Cref{assump:grad_norm_max} essentially requires the gradient norm square cannot exceed its average multiplied by $1/\sqrt{\eta\lambda} \approx 350$, which is reasonable for most iterates.

%

\begin{restatable}[\sgdwd]{theorem}{sgdmain}\label{thm:sgd_main}
Let $\bx(t)$ be defined by \sgd~\eqref{eq:SGD}.  For $\eta\lambda\le 0.1$, under \Cref{assump:grad_norm_max}, with probability $1-5\delta$,
\begin{align}
	\forall T_1\le t\le T-1,\quad \frac{\undersigma^2}{2} \le \frac{2\lambda}{\eta}\norm{\bx(t)}_2^4 \le 4\oversigma^2,
\end{align}

and

\begin{align}
\begin{aligned}\frac{1}{T-T_1}\sum_{t=T_1}^{T-1} \norm{\nabla L(\overline{\bx}(t))}_2^2
\le &\frac{\pi^2\rho\oversigma}{(T-T_1)\sqrt{2\eta\lambda}} + 4\sqrt{\eta\lambda}\frac{\rho\oversigma^3}{\undersigma^2} \\
+  & \sqrt{\frac{\ln\frac{2}{\delta}}{T-T_1}} 4\frac{\pi\rho M \oversigma}{\undersigma} 
+ \sqrt{\frac{\ln\frac{2}{\delta}}{T-T_1}} 4\sqrt{\lambda\eta}\frac{M^2\rho\oversigma}{\undersigma^2},
\end{aligned}
\end{align}
where $T_1 = \frac{1}{4\eta\lambda}\max\left\{\ln \frac{M^2\eta\lambda}{\oversigma^2} + \left\lvert \ln \frac{2e^4M^2}{\norm{\bx(0)}_2^4\eta^{-2}}\right\rvert, 8  \right\}$.	
\end{restatable}

The proof of this theorem is presented in \Cref{sec:sgd}. 
Similar to our earlier result for \gd this bound matches the standard $O(T^{-1/4})$ convergence rate of \sgd for non-convex functions by setting $T=\tilde{O}(\frac{1}{\eta\lambda})$. Further, it only has a logarithmic dependence on the initialization scale $\norm{\bx(0)}_2$, and enjoys robustness to initialization scale as discussed earlier for \gd. We further extend this result to the case where the scale invariant loss has multiple scale invariant parameter groups in \Cref{sec:multi_group}.

We next present our analysis for \sgd with clipping.

\subsection{Convergence of SGD with Relative Global Clipping}\label{sec:main_clip_sgd}

Now we will present our analysis for the clipped \sgd. Recall the clipped \sgd update from Algorithm~\ref{alg:clipped_sgd} has the following norm dynamics.

\textbf{Norm dynamics of clipped SGD:} 
\begin{align*}
\norm{\bx(t+1)}_2^2 = 	(1-\eta\lambda)^2 \norm{\bx(t)}_2^2 
+ \eta^2\min\left\{\frac{\noise{\overline\bx(t)}}{\norm{\bx(t)}_2^2}, \frac{2\lambda C}{\eta} \norm{\bx(t)}_2^2 \right\}.
\end{align*}

\begin{figure*}[t]
\centering
  \includegraphics[width =1.\textwidth]{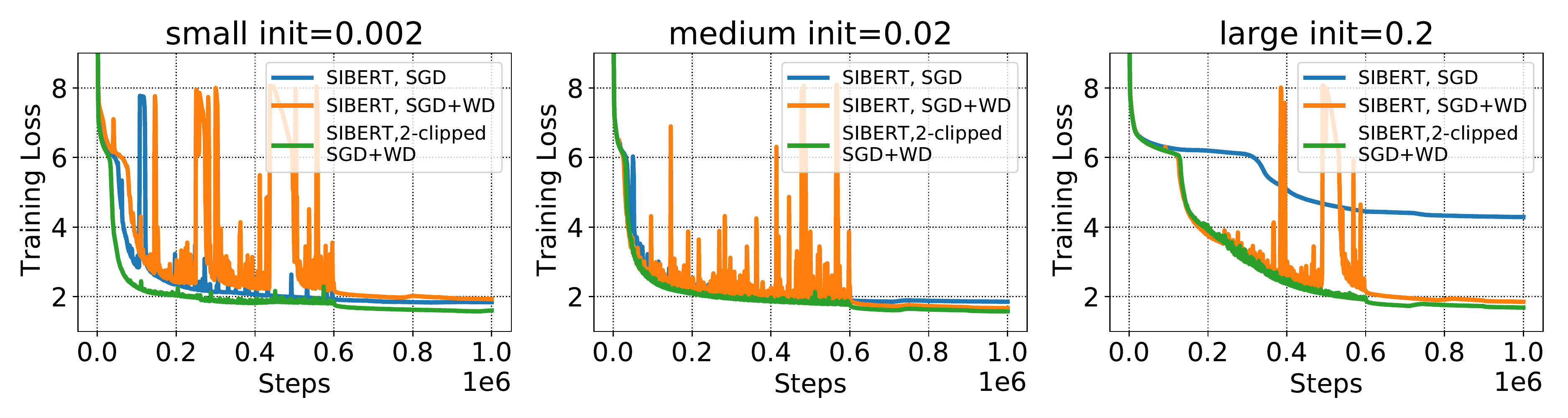}
  \caption{\sgdwd optimizes the scale invariant training loss of \sibert robustly for all initialization scales, and thus for loss scalings and different learning rates (with $\lambda\eta$ fixed). Here the default initialization for parameters in \sibert encoder is a truncated normal distribution with standard deviation equal to $0.02$ (the same as \bert). }
  \label{fig:diff_init}
\end{figure*}

To present our bound we need the following definitions.
\begin{definition}[$C$-clipped mean]
Given a distribution $P$ on $\mathbb{R}_{\ge 0}$ and constant $C>1$, we define $F_{P, C}(\mu)=\EE_{t\sim P}[\min\{t,C\mu\}]$, and define the \emph{$C$-clipped mean} of $ P$, $\mu_{P, C}$ as the largest positive real number satisfying that $F_{P, C}(C\mu_{P, C}) = \mu_{P, C}$. 
Such a definition is valid because $F_{P, C}(0)=0$ and thus $0$ is always a solution. 

For convenience, we also define  $G_{P,C}(\mu): =F_{P, C}(C\mu)-\mu$ and $M_{P, \frac{1}{C}}$ is defined as the $\frac{1}{C}$ median of $P$, that is, $M_{P,C}:=\sup\left\{M\ge 0\mid \PP_{t\sim P}[t\ge M] \ge \frac{1}{C}\right\}$. Since the cumulative density function $\PP_{t\sim P}[t\ge M]$ is left continuous in $M$, it holds that $\PP_{t\sim P}[t\ge M_{P,C}]\ge \frac{1}{C}$. 
\end{definition}

Let $P_{\bx}$ denote the distribution of $\norm{\nabla L_{\gamma}(\bx)}_2^2$. Below is a mild assumption saying $P_{\bx}$ is universally well-concentrated from below in the sense that the mean of the smallest $(1-\frac{1}{C})$ part of $P_{\bx}$ is at least a constant fraction of the $C$-clipped mean of $P_{\bx}$. Since $\mu_{P_{\bx},C}\le \mu_{\bx}$, the assumption below holds whenever $\alpha_C \mu_{\bx}\le \EE_{t\sim P_{\bx}}[t\ind[t<M_{P_{\bx},\frac{1}{C}}]]$.

\begin{assumption}\label{assump:g_max}
 $\exists \alpha_C>0$, such that for all $\bx\neq 0$, $\alpha_C \cdot\mu_{P_{\bx},C}\le \EE_{t\sim P_{\bx}}[t\ind[t<M_{P_{\bx},\frac{1}{C}}]]$.
 \end{assumption}
 
We further define $\mumin:= \min\limits_{\norm{\bx}_2=1} \mu_{P_{\bx},C}$ and $\mumax:= \max\limits_{\norm{\bx}_2=1} \mu_{P_{\bx},C}$ and  have the following theorem:

\begin{restatable}[$\sqrt{C}$-Clipped \sgdwd]{theorem}{clipsgdmain}\label{thm:clip_sgd_main}
Let $\bx(t)$ be defined by $\sqrt{C}$-Clipped \sgd+WD~(\Cref{alg:clipped_sgd}).  Under \Cref{assump:g_max},  for $\eta\lambda=O(\min\{1, \frac{\alpha_C}{C\ln T/\delta^2}\})$, with probability $1-5\delta$, we have
	\begin{align}\label{eq:clip_norm_convergence}
			&\forall T'\le t\le T-1,\quad	\frac{\mumin}{2} \le \frac{2\lambda}{\eta}\norm{\bx(t)}_2^4 \le 2\mumax .
	\end{align}
and 
\begin{align}\label{eq:clipping_sgd}
\begin{aligned}
\frac{1}{T-T'}\sum_{t=T'}^{T-1} \inner{\nabla L(\overline{\bx}(t))}{\widetilde{\nabla L}(\bx(t)) }
\le &\frac{\pi^2\rho\sqrt{\mumax}}{(T-T')\sqrt{2\eta\lambda}} + 4\sqrt{\eta\lambda}\frac{\rho\mumax^{\frac{3}{2}}}{\mumin} \\
+  & \sqrt{\frac{\ln\frac{2}{\delta}}{T-T'}} 8\frac{\pi\rho  \mumax^2}{\mumin} 
+ \sqrt{\frac{\ln\frac{2}{\delta}}{T-T'}} 16\sqrt{\lambda\eta}\frac{\rho\mumax^3}{\mumin^2}.
\end{aligned}
\end{align}
where $T' = \frac{1}{\alpha_C\eta\lambda} \max\left\{\ln \frac{R^2_0}{\mumax}, \ln \frac{ \mumin}{R_0^2}  \right\} +O(1) $ and $\widetilde{\nabla L}(\bx):=\EE\left[\nabla L_{\gamma}(\overline \bx) \min\left\{\sqrt{\frac{2C\lambda}{\eta}}\frac{\norm{\bx}^2_2}{\norm{\nabla L_{\gamma}({\overline \bx})}_2},1\right\}  \right]$.
\end{restatable}
The proof of this theorem is presented in \Cref{sec:proof_clipping}. Note that with clipping \Cref{thm:clip_sgd_main} shows that the norm convergence~\eqref{eq:clip_norm_convergence} is more robust as it doesn't need to make any assumption about the maximum gradient norm $M$, unlike \Cref{thm:sgd_main}. Indeed, from the definition of $C$-clipped mean, for each $\bx$, we can allow all the gradients with norm larger than $C\cdot\mu_{P_{\bx},C}$ to become infinity, and yet not affect the norm convergence, as $\mu_{P_{\bx},C}$ and the condition in \Cref{assump:g_max} do not change.

Under the additional assumption that $\inner{\nabla L(\overline{\bx(t)})}{\widetilde{\nabla L}(\bx(t)}$$=\Omega(\norm{\nabla L(\bx(t))}_2^2)$, we can use \Cref{eq:clipping_sgd} to show convergence to stationary points. This is a reasonable assumption if the clipping frequency is low, \emph{e.g.}, it's $1.5\%$ in our experiments for \sibert.

\begin{figure*}[!t]
  \includegraphics[width =\textwidth]{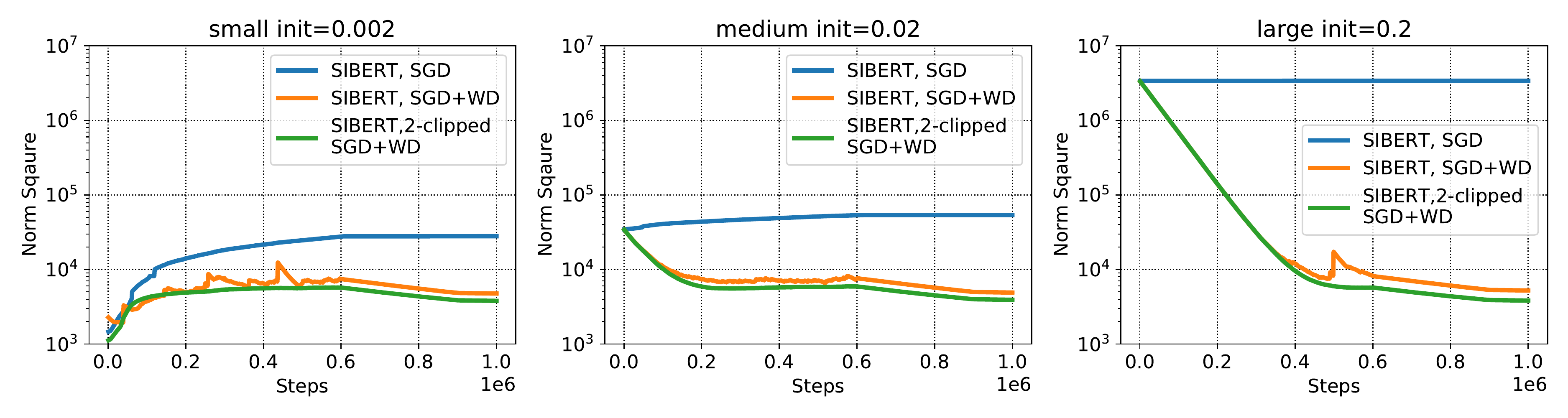}
  \caption{The robust optimization performance of \sgdwd over the scale invariant training loss of \sibert originates from its ability to fast adjust the parameter norm. In contrast, when the initial norm is too large, \sgd w.o. WD optimizes slowly. Relative Global Clipping reduces the spikes in the norm curve, which verifies our theoretical result \Cref{thm:clip_sgd_main} that clipping leads to better norm convergence. Here, only the norm of the scale invariant part, \emph{i.e.}, the encoder part is plotted.}
  \label{fig:norm_convergence}
\end{figure*}

\section{Experiments}\label{sec:exp}
We now conduct a comprehensive empirical study in order to demonstrate the following key aspects of our recipe: 
(i) yields competitive training performance using significantly low memory footprint,
(ii) training becomes highly robust to initialization scale, and
(iii) provides better convergence of norm with clipping.


\paragraph{Experimental Setup.}
We consider the standard task of pretraining a transformer model and fine-tuning it on benchmark datasets, following~\citet{devlin2018bert}.
We compare its performance with \sibert, a scale invariant version of \bert as described in Sec.~\ref{sec:si-design}. For both these models, we use their base size versions unless specified otherwise. For \sibert, the scale invariant portion is trained using SGD+WD with a piecewise constant LR schedule and  WD of $1e-2$. We use \lamb optimizer for the non-scale invariant parts. The initial LR for \sgd is $8e-4$ without warmup and is divided by $10$ at step $600$k and $900$k.  Default training is for $1M$ steps. For \lamb we use a linear decay schedule with initial learning rate $8e-4$ and a linear warmup of $10$k steps. 


\begin{figure}
\centering
  \includegraphics[width=0.6\columnwidth]{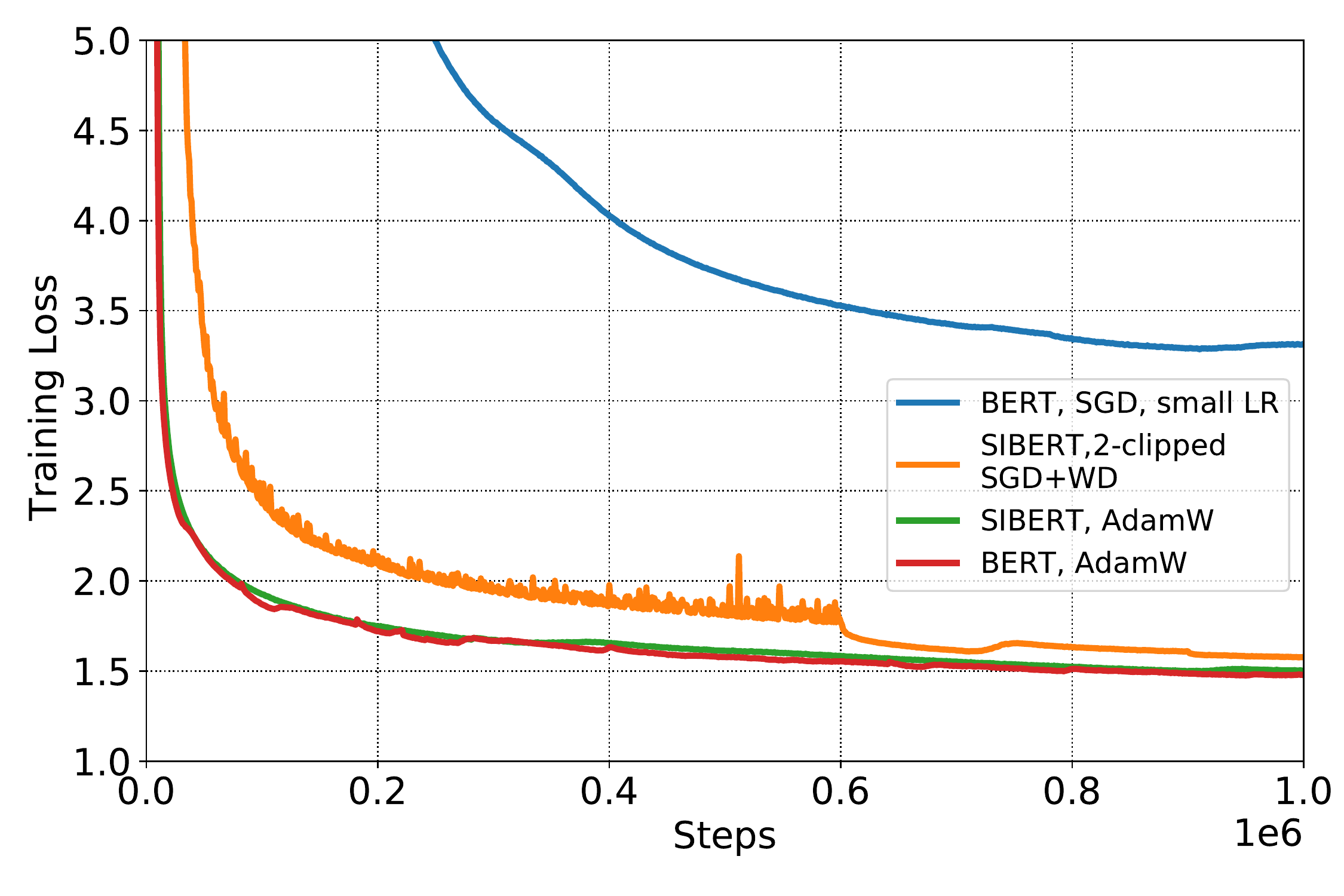}
  \caption{Our recipe (\sibert, \sgdwd and Relative Global Clipping) significantly improves the optimization performance compared to the baseline, \bert trained by \sgd with small LR. The final training loss is close to \bert trained by \adam.}
  \label{fig:compare_to_adam}
\end{figure}

\paragraph{Performance.}
We begin by establishing that proposed \sibert with \sgdwd training performs competitively.
In this regard, we first look at pretraining loss between standard training of \bert with \adam and our \sibert trained by \sgdwd with or without clipping (the clipping factor is set as $\sqrt{C}=2$).
From \Cref{fig:compare_to_adam}, one can see that our training curve closely follows that of \bert trained by \adam, but without the need for extra memory for keeping track of first and second order momentum.
If we use \sgd on standard \bert architecture, then either we have to use small learning rates, which slows down training, or the loss diverges.
This further highlights the importance of the scale invariant architecture, which improves training stability by eliminating the $k$-homogeneous structure. To our knowledge, this is the first work that shows effective training of \bert-like model using simple SGD (even without any momentum).

Next, we compare the downstream performance on three benchmark datasets (SQuADv1.1~\citep{rajpurkar2016squad}, SQuADv2~\citep{rajpurkar2018know} and MNLI~\citep{N18-1101}).
We tried to follow standard setup, e.g. \bert is finetuned by \adam. However for \sibert we had to use LAMB, as \adam is very sensitive to the scale.
We observe comparable performance and when trained longer it can even outperform conventional \bert.


\begin{table}[h]
  \caption{Downstream Performance of \sibert trained by \sgdwd+clipping is close to that of \bert trained \adam - which uses $3X$ more memory than \sgd. The gap is further reduced by doubling the training budget of \sibert.}
  \label{table:downstream_performance}
\small
\centering
  \begin{tabular}{@{}l@{  }llccc@{}}
  \toprule
     &&\!\!\!\small{ MNLI}\!\!\! & {\small SQuAD1}\!\!\! & \small{SQuAD2}  &  Pretraining\\
     && Acc & F1 & F1 & Loss \\
    \midrule
    \multirow{4}{*}{\rotatebox[origin=c]{90}{Base}}&
    \bert    &  \textbf{84.4}    &  \textbf{90.3} &  78.8 & \textbf{1.479}\\
    &\sibert &  81.1& 88.1 &  74.8 & 1.672\\
    &    \  + clipping & 82.6&  89.3& 76.8 & 1.58\\
    &    \  + 2x training  & 83.3 & \textbf{90.3} &  \textbf{80.0}& 1.495 \\
    \midrule
    \multirow{4}{*}{\rotatebox[origin=c]{90}{Large}}&
    \bert    &  \textbf{86.8} &	\textbf{92.4}	& \textbf{84.1} & \textbf{1.181}\\
    &\sibert  & 83.7 &	90.6   & 79.3 & 1.404\\
    &    \  + clipping & 85.3 &	91.6   & 81.3 & 1.322\\
    &    \  + 2x training  & 86.4 & \textbf{92.4} & 83.1 & 1.194  \\
   \bottomrule
  \end{tabular}
  
\end{table}

\paragraph{Training Stability: Insensitivity to the scale of initialization.}

To showcase ease of optimization offered by our recipe, we consider different initialization scales spanning two orders of magnitude.
The results for the pretraining task in ~\Cref{fig:diff_init} show good convergence across the board for our approach, whereas \sgd on its own struggles even with the scale invariant architecture.

Further note that these experiments simultaneously showcase robustness to rescaling of loss, parameterization, or LR.
This is because in a scale invariant model trained by \sgdwd(+clipping), it holds that all of following scalings are equivalent:
$(c_1L,c_2\bx(0),c_3\eta,c_4\lambda)\longleftrightarrow 
(L,\frac{c_2}{\sqrt{c_1c_3}}\bx(0),\eta,c_3c_4\lambda)$ for any $c_1,c_2,c_3,c_4>0$.

\paragraph{Training Stability: Improvement in parameter norm convergence.}
Finally, we look at parameter norms during training in experiments. We observe that even when starting from very different initialization scale, \sgdwd(+clipping) quickly brings parameter norm to desired ranges. 
In contrast, \sgd struggles when initial norm and learning rate are not aligned - see the rightmost plot with large initialization in \Cref{fig:norm_convergence}.
This shows that our recipe has the ability to quickly adapt to different initialization scales, in-line with our theoretical result (\Cref{thm:clip_sgd_main}) showing better norm convergence of \sgdwd(+clipping).

\section{Conclusion}


In this paper, we presented a simple yet effective method to robustly train transformers with non-adaptive methods such as \sgd. By designing novel scale invariant architecture and using a tailored optimization procedure --- which makes our optimization scheme truly \emph{architecture aware} --- we provably achieve robust training of neural networks with substantially low memory footprint when compared to adaptive methods. We believe designing neural architecture and the optimizer jointly is an exciting research direction and will yield even better training procedures in the future.


\bibliography{zhiyuan}
\bibliographystyle{icml2022}

\newpage
\appendix
\onecolumn

\section{Design Details of Scale Invariant BERT}\label{sec:SI_design}

\begin{definition}
For a module with $n$ inputs and $m$ outputs, we say the module is $(a_1,...a_n; b_1,...,b_m)$-homogeneous if the $m$ outputs are $b_i$-homogeneous to the network parameters whenever the $n$ inputs are $a_i$-homogeneous to the network parameters. A model is scale invariant iff its output is $(;0)$-homogeneous. (A complete model doesn't take any input from another module)
\end{definition}

Following \cite{li2019exponential}, we view the computation graph as a directed acyclic graph, where each module is a node and each tensor (including inputs, intermediate computation results and final output) as an edge. Each edge can be viewed as a function of parameters, and  we can decide the homogeneity by doing induction over the computation graph by its topological order. In detail, we know the $j$th output edge of some $(a_1,\ldots,a_n;b_1,ldots,b_n)$- homogeneous module is $b_j$ homogeneous if for each $1\le i\le n$, the $i$th input edge is $a_i$-homogeneous.  For convenience, we allow $a_i$,$b_i$ to be functions of free variable $x$, meaning the module is $(a1(x),\ldots,a_n(x); b_1(x),\ldots,b_m(x))$-homogeneous for every $x\in\RR$.

In \cref{tab:module_homo}, we summarize the homogeneity of building blocks in our design.

\paragraph{Overview of SIBERT structure:} Our SIBERT has two main parts --- encoder and classification head, which is the same to standard BERT. We only make encoder part scale invariant and train it by \sgdwd. We leave the classification head not scale invariant and train it by \lamb. Note the classification head is only used in pretraining and is not used in the downstream task.

\paragraph{(2;2)-homogeneous encoder layer:} As mentioned in \Cref{sec:SI_design}, residual block and attention are the two main building blocks that needs to be made scale invariant. Following \citet{li2019exponential}, we choose to use PreNorm structure for residual block and make it $(2;2)$-homogeneous.  We also replace GeLU~\cite{hendrycks2016gaussian} in BERT by ReLU for homogeneity. Since ReLU is $(1;1)$ homogeneous, we omit ReLU from the design, without affecting the final scale invariance.

\begin{table}[!htbp]
    \centering
    \caption{Homogeneity of building blocks of SIBERT.}
    \begin{tabular}{c|c|c}
 Symbol &Module &  Homogeneity \\
 \hline
I   &  Input  &(0;1) \\
B   & Adding Bias & (1;1)\\
N & Layer Normalization (no affine) &(x;0)\\
L &Linear Layer &(x;x+1)\\
Embed &Embedding Layer &(x;x+1)\\
NA & Layer Normalization  with affine  &(x;1)\\
FF & 2-layer feedforward network  &(0;2)\\
ATTN & {\color{violet}Scale Invariant Attention} &(x,x,x;x+2)\\
Encoder & {\color{blue}Our Encoder Layer} &(2;2)
    \end{tabular}
    \label{tab:module_homo}
\end{table}

\begin{figure}[!htbp]
    \centering
    \includegraphics[width=0.8\textwidth]{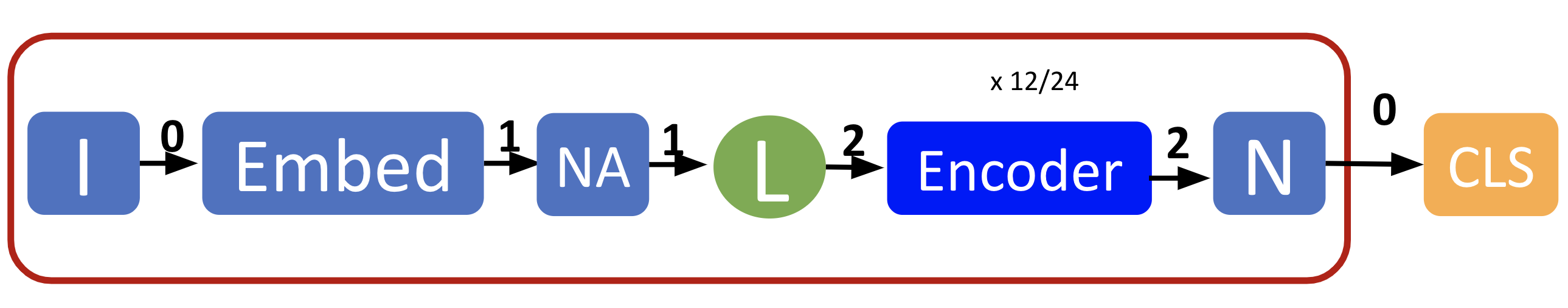}
    \caption{{\color{red} Encoder} and {\color{orange} Classification Head (CLS)}. `x12/24' means to stack $12$ our $(2;2)$-homogeneous encoder layer for base SIBERT (or 24 for large SIBERT)  }
    \label{fig:SI_overview}
\end{figure}

\begin{figure}[!htbp]
    \centering
    \includegraphics[width=0.8\textwidth]{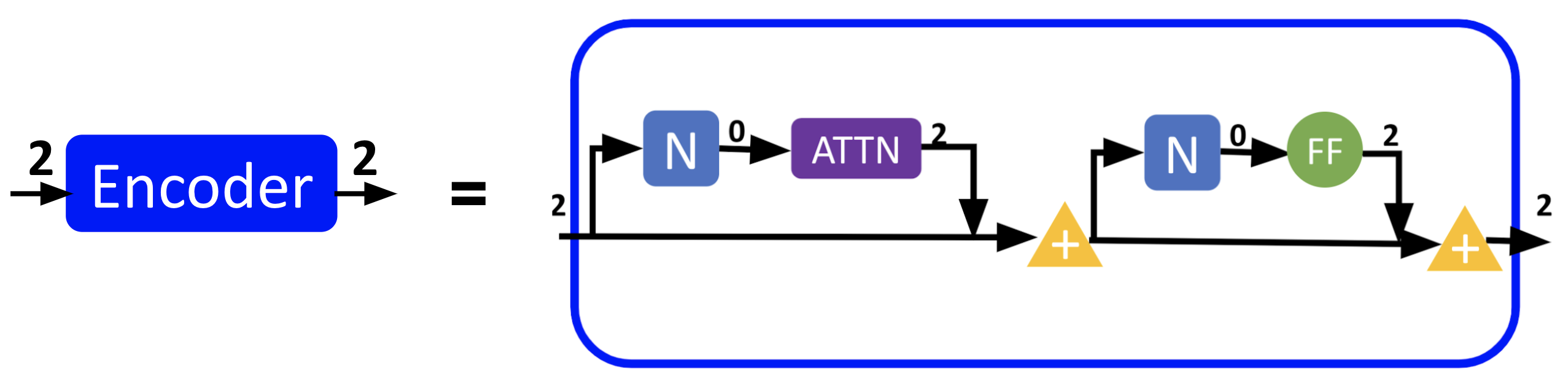}
    \caption{The $(2;2)$-homogeneous encoder layer. `ATTN' denotes our Scale Invariant Attention~(see \Cref{fig:SI_attn}). `FF' denotes the 2-layer feedforward structure, which is $(0;2)$-homogeneous.}
    \label{fig:SI_encoder}
\end{figure}

\begin{figure}[!htbp]
    \centering
    \includegraphics[width=0.7\textwidth]{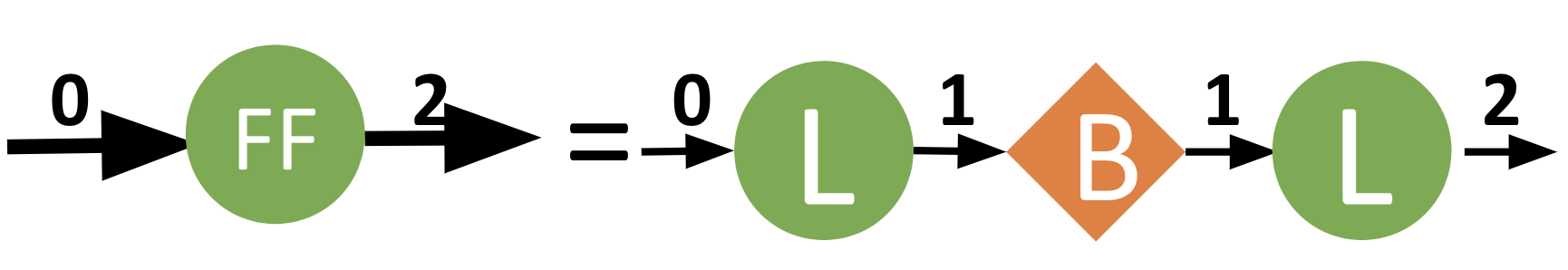}
    \caption{The $(0;2)$-homogeneous FeedForward layer}
    \label{fig:SI_FF}
\end{figure}

\begin{figure}[!hbt]
\centering
  \includegraphics[width = 0.78\textwidth]{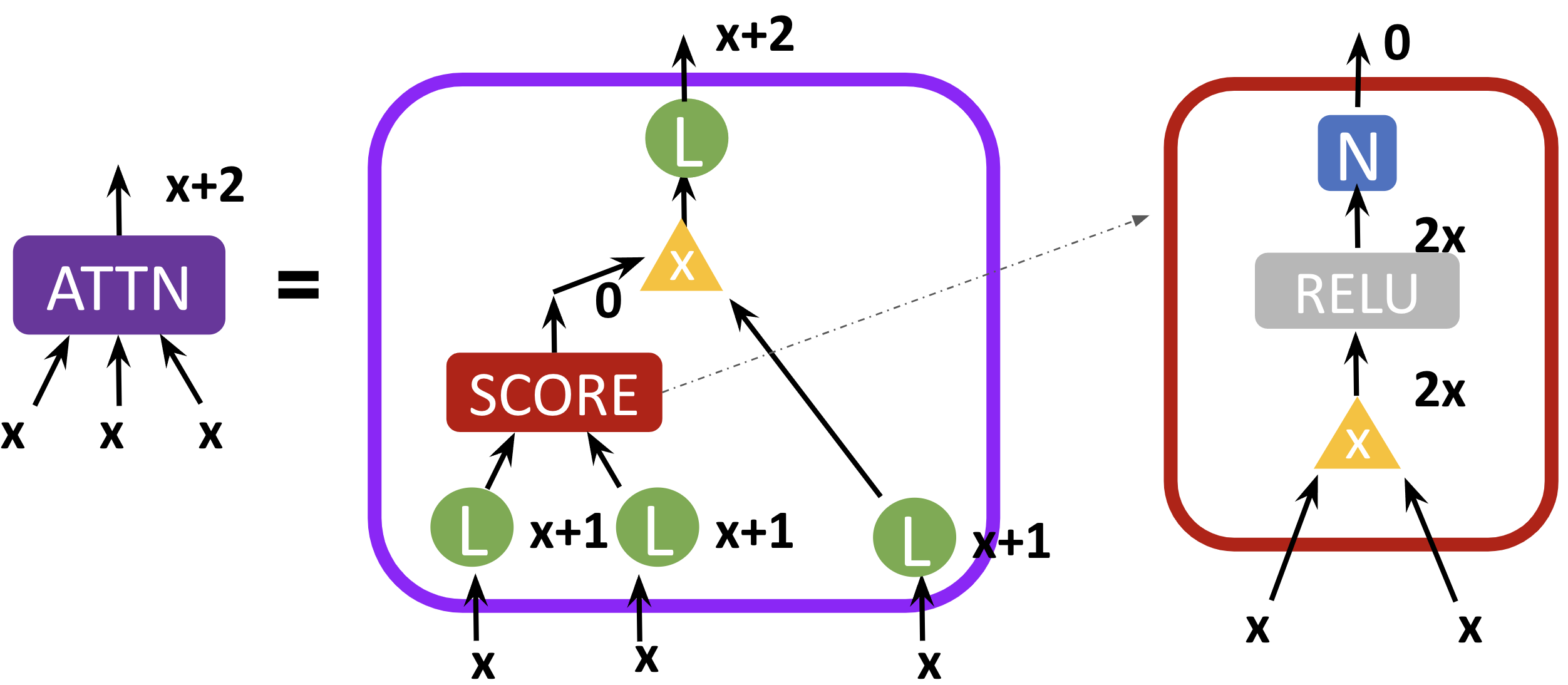}
  \caption{The $(x,x,x;x+2)$-homogeneous Attention, which is defined as $\musiattn(Q,K,V) = \sum_i \normalization(\relu(Q W_i^Q (KW_i^K)^\top) VW_i^V  W_i^O$, where  $W_i^Q, W_i^K$ $\in \RR^{d_{ model}\times d_k}$, $W_i^V\in \RR^{d_{ k}\times d_v}$ and $W_i^O\in \RR^{d_{ v}\times d_{model}}$ That is, if $Q,K,V$ are $k$-homogeneous functions of parameter $\bx$, then $\musiattn(Q,K,V)$ is $k+2$-homogeneous, for any $k\in\mathbb{R}$. We also call it \emph{Scale Invariant Attention} because its attention score is scale invariant.  }\label{fig:SI_attn}
\end{figure}

\newpage
 \section{Introduction examples analysis}
 \label{sec:intro_analysis}
 
 In the first example, since the data is non-separable, the global optimum $X^*$ must be finite  and, thus, $|\nabla \tilde{L}(X)|$ is positive  and monotone increases among all $X>X^*>0$. For simplicity, assume  $X^*>0$ and $x_1=\cdots=x_{2k}> ({X^*})^{\frac{1}{2k}}$ at initialization (and thus at any iteration $t$). It holds that $x_i(t+1) = x_i(t)-\eta \frac{X(t)}{x_i(t)} \nabla \tilde{L}(X(t))  = x_i(t) \left(1-\eta\frac{X(t)}{x_i^2(t)} \nabla \tilde{L}(X(t))\right)$, where $X(t)=\Pi_{j=1}^{2k} x_j(t)$. This implies $X(t+1)=X(t)\left(1-\eta\frac{X(t)}{\sqrt[k]{X(t)}}\nabla \tilde{L}(X(t))\right)^{2k}\ge 0$. Thus we conclude if $\eta \ge \frac{2}{|\nabla \tilde{L}(X(0))|}(X(0))^{\frac{1}{k}-1}$ and $X(0)>X^*$, $X(t)$ will increase monotonically and explode.
\section{Useful Lemmas}

\subsection{Scale Invariance}

\begin{lemma}[Smoothness]\label{lem:taylor_expansion}
For any $\bv,\bx\in \RR^d$ with $\inner{\bx}{\bv}=0$, suppose $L$ is scale-invariant and twice differentiable with $\rho:=\max_{\norm{\bx}_2=1} \norm{\nabla^2 L(\bx)}$, we have
\begin{align*}
    L(\bx+\bv) - L(\bx) \le \inner{\bv}{\nabla L(\bx)} + \frac{\rho\norm{\bv}_2^2}{2\norm{\bx}_2^2}.
\end{align*}
\end{lemma}
\begin{proof}[Proof of \Cref{lem:taylor_expansion}]
Define $\gamma(s) = \bx + s\bv$, then we have $L(\gamma(0)) = L(\bx)$ and $L(\gamma(1)) = L(\bx+\bv)$. Taking Taylor expansion of $F(s)=L(\gamma(s))$ at $s=0$, we have 
\begin{align*}
F(1)-F(0) = F'(0) + \frac{F''(s^*)}{2},\quad \textrm{for some } s^*\in[0,1].
\end{align*}

Note $F'(0) = \inner{\gamma'(0)}{\nabla L(\gamma(0))}=\inner{\nabla L(\bx)}{\bv}$ and 
\begin{align*}
F''(s^*)= &\gamma'(s^*) \nabla^2L(\gamma(s^*))\gamma'(s^*) 
\le \frac{\rho}{\norm{\gamma(s^*)}_2^2} \norm{\gamma'(s^*)}_2^2,
\end{align*}
where the last inequality uses the fact that $L$ is scale invariant.
The proof is completed by noting that $\norm{\gamma(s^*)}_2\ge \norm{\gamma(0)}_2 = \norm{\bx}_2^2$ and that $\gamma'(s^*)= \bv$.

\end{proof}

\begin{lemma}[Smoothness, Multi-group]\label{lem:taylor_expansion_multi}
For any $\bv,\bx\in \RR^d$ with $\inner{\bx_k}{\bv_k}=0$ for all $k\in [K]$, suppose $L$ is multi-group scale invariant (see \Cref{defi:multi_group_scale_invariance}), we have
\begin{align*}
    L(\bx+\bv) - L(\bx) \le \inner{\bv}{\nabla L(\bx)} + \frac{\rho}{2}\sum_{k=1}^K\frac{\norm{\bv_i}_2^2}{\norm{\bx_i}_2^2}.
\end{align*}
\end{lemma}
\begin{proof}[Proof of \Cref{lem:taylor_expansion_multi}]
We first prove for the case where $\norm{\bx_k}_2=1$, $\forall k\in[K]$. Similar to the proof of \Cref{lem:taylor_expansion}, it suffices to show that the smoothness of $L$ is at most $\rho$ along the line joining $\bx$ and $\bx+\bv$. This holds because $\forall s\in[0,1],k\in[K]$, $\norm{\bx_i+s\bv_i}_2\ge \norm{\bx_i}_2$ by assumption that $\inner{\bx_k}{\bv_k}=0$ for all $k\in [K]$. 

Now we turn to the general case. Define $\hat{\bx} = [\frac{\bx_1^\top}{\norm{\bx_1}_2}, \ldots, \frac{\bx_K^\top}{\norm{\bx_K}_2}]^\top$ and $\bv'  = [\frac{\bv_1^\top}{\norm{\bx_1}_2}, \ldots, \frac{\bv_K^\top}{\norm{\bx_K}_2}]^\top$. Since $L$ is multi-group scale invariant, we have $L(\bx) = L(\hat{\bx})$ and $L(\bx+\bv) = L(\hat{\bx}+\bv')$. The proof is completed by applying the previous argument on $\hat{\bx}$ and $\bv'$.
\end{proof}

\begin{lemma}\label{lem:gradient_upper_bound}
If $L$ is scale invariant, $\norm{\nabla L(\bx)}_2 \le \frac{\pi}{\norm{\bx}_2} \sup_{\norm{\bx}=1}\norm{\nabla^2 L(x)}_2$. 
\end{lemma}

\begin{proof}[Proof of \Cref{lem:gradient_upper_bound}]
	It suffices to prove the above bound for all $\bx$ with $\norm{\bx}_2 =1$. Let $\bx^*$ be any local minimizer of $L$ on $\SSS ^{d-1}$ and $\gamma:[0,1]\to \SSS^{d-1}$ be the geodesic curve satisfying that $\gamma(0) = \bx^*$ and $\gamma(1) = \bx$. We know the length of $\{\gamma(t)\}_{t=0}^1\le \pi$ and thus
	\begin{align*}
	\norm{\nabla L(\bx)} = \norm{\int_{t=0}^1 \nabla^2 L(\gamma(t))\frac{d\gamma (t)}{dt} dt} \le 	\int_{t=0}^1 \norm{ \nabla^2 L(\gamma(t))}_2 \norm{\frac{d\gamma (t)}{dt}}_2 dt \le \rho\cdot \pi 
	\end{align*}

\end{proof}

\begin{lemma}\label{lem:loss_upper_bound}
If $L$ is scale invariant,  $\sup_{\bx,{\bx}'} L(\bx) - L(\bx') \le \frac{\pi^2}{2}\sup_{\norm{\bx}=1}\norm{\nabla^2 L(x)}_2$.
\end{lemma}

\begin{proof}[Proof of \Cref{lem:loss_upper_bound}]
Similar to the proof of \Cref{lem:gradient_upper_bound}. 
\end{proof}

\subsection{Probablity}
\begin{definition}
A random variable $X\in \RR$ is said to be \emph{sub-Gaussian} with variance proxy $\sigma^2$ (denoted by $X\sim \subg(\sigma^2)$) if its moment generating function satisfies 
\begin{align*}
\E[\exp(sX)] \le \exp(\frac{\sigma^2s^2}{2}), \forall s\in \RR.	
\end{align*}	

In this work, we also use the following notion of \emph{conditional subgaussian}. We say a random variable $X\in\RR$ is said to be \emph{sub-Gaussian} with variance proxy $\sigma^2$ conditioned on event $\cE$ (denoted by $X\sim \subg(\sigma^2,\cE)$) if its moment generating function satisfies 
\begin{align*}
\E[\exp(sX)\ind[\cE] ] \le \exp(\frac{\sigma^2s^2}{2}), \forall s\in \RR.	
\end{align*}	
\end{definition}

\begin{lemma}[Chernoff Bound with Conditioning]\label{lem:subgaussian}
Let $X\sim \subg(\sigma^2,\cE)$. Then for any $t>0$, it holds that 
	\begin{align*}
		\PP[X>t\wedge \cE] \le \exp(-\frac{t^2}{2\sigma^2})	, \quad \textrm{and} \quad 	
		\PP[X<-t \wedge \cE] \le \exp(-\frac{t^2}{2\sigma^2})
	\end{align*}
	
	When $\PP[\cE]=1$, we get the standard Chernoff bound. Let $X\sim \subg(\sigma^2)$. Then for any $t>0$, it holds that 
	\begin{align*}
		\PP[X>t] \le \exp(-\frac{t^2}{2\sigma^2})	, \quad \textrm{and} \quad 	
		\PP[X<-t] \le \exp(-\frac{t^2}{2\sigma^2})	
	\end{align*}
\end{lemma}

\begin{proof}[Proof of \Cref{lem:subgaussian}]
	For any $s>0$, we have 
	\begin{align*}
		\PP[X>t\wedge \cE] = \PP[e^{s X}\ge e^{s t}\wedge \cE] \le e^{-s t}\EE[e^{s X} \ind[\cE]]	= \exp(-st+\frac{\sigma^2s^2}{2}).
	\end{align*}
	The proof is completed by picking $s=\frac{t}{\sigma^2}$.
\end{proof}

We will use $(\Omega, \Sigma, \PP)$ to note the probability space and $\{\cF_t\}_{t\in \NN}$ to denote the filtration.
\begin{lemma}[Azuma Inequality with Conditioning]\label{lem:hp_gaussian}
Let $\cE_t\in \cF_t$ and $\cE_{t+1}\subset\cE_{t}$ for all $t\ge 0$. Let $\{X_t\}_{t\ge 1}$ be a  martingale difference sequence and $\subg(\sigma_t^2,\cE_{t-1})$ conditioned on $\cF_{t-1}$,  \emph{i.e.}, $\EE[\exp(sX_t) \ind[\cE_{t-1}]\mid\cF_{t-1}]\le \exp(\frac{s^2\sigma_t^2}{2})$ for all $t\ge 0$. Then $\sum_{i=1}^{T}X_i$ is $\subg(\sum_{t=0}^{T-1}\sigma_t^2,\cE_{T-1})$.
\end{lemma}

\begin{proof}
We will prove by induction on $T$. When $T=1$, the statement is true by assumption. Now suppose the statement holds for $T-1$, we have for any $s>0$ 
\begin{align*}
    \EE[\exp(s\sum_{i=1}^{T}X_i)\ind[\cE_{T-1}]] 
    = &\EE[\exp(s\sum_{i=1}^{T-1}X_i)\ind[\cE_{T-1}]\EE[\exp(sX_T) \ind[\cE_{t-1}]\mid \cF_{T-1}]]\\
    \le &\EE[\exp(s\sum_{i=1}^{T-1}X_i)\ind[\cE_{T-1}]\exp(\frac{s^2\sigma_{T-1}^2}{2})]\\
    \le &\EE[\exp(s\sum_{i=1}^{T-1}X_i)\ind[\cE_{T-2}]]\exp(\frac{s^2\sigma_{T-1}^2}{2})
\end{align*}
Thus we have that
 $   \EE[\exp(s\sum_{i=1}^{T}X_i)\ind[\cE_{T-1}]]  \le \exp(\frac{s^2\sum_{t=0}^{T-1}\sigma_t^2}{2}).$
\end{proof}

\subsection{Others}
\begin{lemma}\label{lem:power_series}	$\forall t\in \NN, k\in \NN^+, 0<x<1,$
\[ \sum_{\tau=0}^t (1-x)^{k\tau} \le  \frac{e^{kx}}{kx}\]
\end{lemma}
\begin{proof}[Proof of \Cref{lem:power_series}]
\[\sum_{\tau=0}^t (1-x)^{k\tau} \le \sum_{\tau=0}^\infty (1-x)^{k\tau} \le  \sum_{\tau=0}^\infty e^{-kx\tau} = \frac{1}{1-e^{-kx}} \le \frac{e^{kx}}{kx},\]
where the last step is because $e^x\ge 1+x$, $\forall x\in\RR$.	
\end{proof}

\section{Omitted Proofs for the Convergence of GD}

\begin{proof}[Proof of \Cref{lem:GD_descent}]
%
This is a special case of \Cref{lem:taylor_expansion} with $\bx=(1-\eta\lambda)\bx(t)$ and $\bv=-\eta \nabla L(\bx(t))$. Here we use the assumption that $L$ is scale invariant, $\nabla L$ is $-1$-homogeneous. By \Cref{lem:homo}, which means $\nabla L(\bx) = \frac{\nabla L(\bx(t))}{1-\eta\lambda}$.
\end{proof}

The following lemma deals with the case where $\norm{\bx(0)}_2^2< \pi^2\rho\eta$. 
\begin{lemma}\label{lem:gd_small_init}
 Let $I = \{T' \in \mathbb{N} \mid \forall 0\le t\le T', \ \norm{\bx(t)}_2^2 \le \pi^2\rho\eta\ \wedge\  \norm{\nabla L(\overline{\bx}(t))}_2^2>8\pi^4\rho^2\lambda\eta\}$. Suppose $0\in I$ and $T =\max I$. Then $T\le \frac{1}{6\lambda\eta}$ and $ \normxx{T+1} \le \frac{2(\pi^2\rho\eta)^2}{\normxx{0}} $.

\end{lemma}

\begin{proof}[Proof of \Cref{lem:gd_small_init}]
For any $t\le T$, we have 
	\begin{align*}
\norm{\bx(t+1)}_2^2-\norm{\bx(t)}_2^2 = &((1-\lambda\eta)^2-1)\norm{\bx(t)}_2^2 + \eta^2 \norm{\nabla L(\bx(t))}_2^2 \\
\ge &-2\lambda\eta \norm{\bx(t)}_2^2 +  \frac{\eta^2 \norm{\nabla L(\overline{\bx}(t))}_2^2}{\norm{\bx(t)}_2^2} \\
\ge &-2\pi^2\rho \lambda\eta^2 + 8\pi^2\rho\lambda\eta^2\\
= & 6\pi^2\rho \lambda\eta^2. 
\end{align*}
	
Thus $6\pi^2\rho \lambda\eta^2\cdot T \le \norm{\bx(T)}_2^2 -\norm{\bx(0)}_2^2< \norm{\bx(T)}_2^2 \le \pi^2\rho\eta$, which implies that $T< \frac{1}{6\lambda\eta}$. Moreover, we have that  
\begin{align*}
    \norm{\bx(T+1)}_2^2 = & (1-\eta\lambda)^2\normxx{T} +\eta^2\gnormxx{T}\\
     \le &\normxx{T} + \frac{\eta^2 \norm{\nabla L(\overline{\bx}(T))}_2^2}{\norm{\bx(T)}_2^2}   \\
     \le &\normxx{T} + \frac{\eta^2 \norm{\nabla L(\overline{\bx}(T))}_2^2}{\norm{\bx(0)}_2^2}   \\
    \le &\pi^2\rho\eta + \frac{\rho^2\pi^2\eta^2}{\normxx{0}}\\	
    \le & \frac{2(\pi^2\rho\eta)^2}{\normxx{0}}.
\end{align*}
This completes the proof.
\end{proof}

\begin{theorem}[convergence rate of GD+WD]\label{thm:GD_main}
Suppose $\eta\lambda\le\frac{1}{2}$. Let $\bx(t)$ be the $t$-th iterate of GD~\eqref{eq:GD}, and $T_0 = \left\lceil \frac{1}{2\eta\lambda}\ln \frac{2\norm{\bx(0)}_2^2}{\rho\pi^2\eta}\right\rceil$. If $\norm{\bx(0)}_2^2\ge \pi^2\rho\eta$,  we have 
\begin{align*}
    \min_{t=0, \ldots, T_0} \norm{\nabla L(\overline\bx(t))}_2^2 \le 8\pi^4\rho^2\lambda\eta.
\end{align*}
\end{theorem}

\begin{proof}[Proof of \Cref{thm:GD_main}]
We first claim there's $0\le t\le T_0$, such that $\norm{\bx(t)}_2^2< \pi^2 \rho\eta$.

Otherwise, by \Cref{lem:GD_descent}, for $t=0,\ldots, T_0$, we have  $ L(\bx(t)) - L(\bx(t+1)) \le \frac{\eta}{2} \norm{\nabla L(\bx(t))}_2^2$. Note that $\norm{\bx(t+1)}_2^2 - (1-\eta\lambda)^2\norm{\bx(t)}_2^2 = \eta^2 \norm{\nabla L(\bx(t))}_2^2$. 

Therefore, we have that 
\begin{align*}
\norm{\bx(T_0)}_2^2 - (1-\eta\lambda)^{2T_0}\norm{\bx(0)}_2^2 
=&\sum_{t=0}^{T_0-1} \eta^2 (1-\eta\lambda)^{2(T_0-t)} \norm{\nabla L(\bx(t))}_2^2	 \\
\le&\sum_{t=0}^{T_0-1} \eta^2  \norm{\nabla L(\bx(t))}_2^2	 \\
\le &\frac{\eta}{2}(L(\bx(0))-L(\bx_{T_0-1}))\\
\le & \frac{\eta \pi^2\rho}{2}
\end{align*}

By the definition of $T_0$, we have $(1-\eta\lambda)^{2T_0}\norm{\bx(T_0)}_2^2  \le e^{-2\eta\lambda T_0}\norm{\bx(0)}_2^2 \le \frac{\eta\pi^2 \rho}{2}$. Thus $\norm{\bx(T_0)}\le \pi^2\rho \eta$.

Without loss of generality, we let $T$ be the smallest integer such that $\norm{\bx(T)}_2^2< \pi^2 \rho\eta$. By assumption, $T\ge 1$. Therefore $\norm{\bx(T-1)}_2^2\ge \pi^2 \rho\eta$. Because $\norm{\bx(T)}_2^2 = (1-\eta\lambda)^2\norm{\bx(T-1)}_2^2 + \eta^2 \norm{\nabla L(\bx(T-1))}_2^2$, we have that
\begin{align*}
	\norm{\nabla L(\overline{\bx}(T-1))}_2^2  
	= &\norm{\nabla L(\bx(T-1))}_2^2\norm{\bx(T-1)}_2^2 
	\le  \eta^{-2} \left( \norm{\bx(T)}_2^2 - (1-\eta\lambda)^2\norm{\bx(T-1)}_2^2\right)) \norm{\bx(T-1)}_2^2.
\end{align*}

Note that $\norm{\bx(T)}_2^2< \pi^2 \rho\eta$ and $\frac{\norm{\bx(T)}^2_2}{(1-\lambda\eta)^2} \ge \norm{\bx(T-1)}_2^2\ge \pi^2 \rho\eta$, we conclude that
\begin{align*}
	\norm{\nabla L(\overline{\bx}(T-1))}_2^2 
	\le & \eta^{-2} \left( \norm{\bx(T)}_2^2 - (1-\eta\lambda)^2\norm{\bx(T-1)}_2^2\right)) 
	\frac{\norm{\bx(T)}_2^2}{(1-\lambda\eta)^2}\\
	\le & \frac{1-(1-\lambda\eta)^2}{\eta^2(1-\lambda\eta)^2} (\pi^2\rho\eta)^2\\
	\le & 8\lambda\eta \pi^4\rho^2,
\end{align*}
 which completes the proof.
\end{proof}

Combining \Cref{lem:gd_small_init} and \Cref{thm:GD_main} removes the initial condition in \Cref{thm:GD_main}, and completes the proof of \Cref{thm:GD_main_2}.

\section{Omitted Proofs for Convergence Rate of \sgd}\label{sec:sgd}

We will use $(\Omega, \Sigma, \PP)$ to note the probability space and $\{\cF_t\}_{t\in \NN}$ to denote the filtration where $\cF_t:=\sigma(\{\gamma_i\mid 0\le i\le t\})$ is the $\sigma$-algebra generated by $\gamma_0,\ldots, \gamma_t$.

\begin{lemma}\label{lem:hoeffding_lemma}
	$\noise{\bx}-\E \noise{\bx}\sim \subg(\frac{M^4}{4\norm{\bx}_2^4})$. 
\end{lemma}

\begin{proof}\Cref{lem:hoeffding_lemma}
	Note $0 \le \noise{\bx} \le \frac{M^2}{\norm{\bx}_2^2}$. The proof is immediate by Hoeffding Lemma (see Lemma 3.6 in \cite{vanprobability}).
\end{proof}

Given a integer $T\ge 0$, let $\mathcal{E}_T$ be the event that  $\forall 0 \le t'\le  t \le T-1,$
\begin{align}\label{eq:ET}
\left| \sum_{\tau=t'}^t (1-\eta\lambda)^{4(t-\tau)}\left(\noisebk{\tau}-\E[\noisebk{\tau}\mid \overline\bx(\tau)]\right)\right| \le  e^{4\eta\lambda}\cdot\frac{M^2}{4}\sqrt{\frac{1}{\lambda\eta}\ln\frac{2T^2}{\delta}}.
\end{align}

\begin{lemma}\label{lem:whp_ET}
	For any $0\le t'\le t\le T-1$, 
	\begin{align*}
	\sum_{\tau=t'}^t (1-\eta\lambda)^{4(t-\tau)}\left(\noisebk{\tau}-\E[\noisebk{\tau}\mid \bx(\tau)]\right) \sim \subg(\frac{e^{8\eta\lambda}M^4}{32})	
	\end{align*}
Thus we have $\PP[\mathcal{E}_T]\ge 1-\delta$ by \Cref{lem:subgaussian}.
\end{lemma}

\begin{proof}[Proof of \Cref{lem:whp_ET}]
	Note that $\sum_{\tau=t'}^{t}(1-\eta\lambda)^{8(t-\tau)}\frac{M^4}{4}\le \frac{e^{8\eta\lambda}}{32}$ by \Cref{lem:power_series}. Thus by Azuma Inequality and \Cref{lem:hoeffding_lemma}, we have that the martingale \[\sum_{\tau=t'}^t (1-\eta\lambda)^{4(t-\tau)}\left(\noisebk{\tau}-\E[\noisebk{\tau}\mid \bx(\tau)]\right) \] is $\frac{e^{8\eta\lambda}}{32}$-subgaussian.
	
	By \Cref{lem:subgaussian}, we have for any  $\forall 0 \le t'\le  t \le T-1$, \Cref{eq:ET} holds with probability at least $\frac{\delta}{T^2}$. The proof is completed by applying union bound.
\end{proof}

\begin{lemma}[Norm Lower Bound]\label{lem:norm_lower_bound}
Under \Cref{assump:grad_norm_max} and additionally assume $\eta\lambda\le \frac{1}{2}$.
On $\mathcal{E}_T$, it holds that for any $t\ge 0$,
\begin{align}
\eta^{-2}\norm{\bx(t)}_2^4 \ge 	\frac{1-\eta\lambda}{2\eta\lambda}(1-e^{-4t\eta\lambda(1-\eta\lambda)})\undersigma^2
-\frac{1}{2}(1-\eta\lambda)^2 M^2e^{4\eta\lambda}\sqrt{\frac{1}{\lambda\eta}\ln\frac{2T^2}{\delta}}
\end{align}

When $\frac{\undersigma^2}{12\eta\lambda} \ge \frac{M^2}{2}e^{4\eta\lambda}\sqrt{\frac{1}{\lambda\eta}\ln\frac{2T^2}{\delta}}$, the above condition is simplified into the following: on $\mathcal{E}_T$ for any $\frac{1}{\eta\lambda}\le t\le T$,
\begin{align}
\eta^{-2}\norm{\bx(t)}_2^4 \ge \frac{5(1-\eta\lambda)^2 \undersigma^2}{12\eta\lambda} - \frac{(1-\eta\lambda)^2 \undersigma^2}{6\eta\lambda}=\frac{(1-\eta\lambda)^2 \undersigma^2}{4\eta\lambda},
\end{align}

In the above inequality, we also used the fact that $1-e^{-4(1-\eta\lambda)}\ge \frac{5}{6}$, which is implied by $\eta\lambda \le 0.5$.
\end{lemma}
\begin{proof}[Proof of \Cref{lem:norm_lower_bound}]	
Since $L_\gamma$ is scale invariant, by \Cref{thm:euler_homo}, we have
\begin{align}\label{eq:norm_dynamics}
\norm{\bx(t+1)}_2^2 = 	(1-\eta\lambda)^2 \norm{\bx(t)}_2^2 + \eta^2\frac{\noisebk{t}}{\norm{\bx(t)}_2^2}.
\end{align}
Squaring both sides of \Cref{eq:norm_dynamics}, we have
\begin{align}\label{eq:norm_dynamics_2}
\norm{\bx(t+1)}_2^4 = 	(1-\eta\lambda)^4 \norm{\bx(t)}_2^4 + 2(1-\eta\lambda)^2 \eta^2\noise{\overline\bx(t)}+ \frac{\eta^4\noisebkf{t}}{\norm{\bx(t)}_2^4}.
\end{align}

Thus 
\begin{align*}
\begin{aligned}
	\eta^{-2}\norm{\bx(t+1)}_2^4 
	\ge & 2 \sum_{\tau=0}^t (1-\eta\lambda)^{4(t-\tau)+2}\noisek{\tau}\\
	\ge & 2 \sum_{\tau=0}^t (1-\eta\lambda)^{4(t-\tau)+2}\E\noisek{\tau}\\ 
	+ & 2 \sum_{\tau=0}^t (1-\eta\lambda)^{4(t-\tau)+2}\left(\noisek{\tau} - \E\noisek{\tau}\right).
\end{aligned}
\end{align*}

We also have that
\begin{align*}
\sum_{\tau=0}^{t} (1-\eta\lambda)^{4(t-\tau)}
\ge \sum_{\tau=0}^{t} e^{-4(t-\tau)\eta\lambda(1-\eta\lambda)}
=\frac{1-e^{-4t\eta\lambda(1-\eta\lambda)}}{1-e^{-4\eta\lambda(1-\eta\lambda)}}
\ge \frac{1-e^{-4t\eta\lambda(1-\eta\lambda)}}{4\eta\lambda(1-\eta\lambda)}.\end{align*}

 
Therefore, it holds that for any $t\ge 0$, conditioned on $\mathcal{E}_T$,
\begin{align*}
\eta^{-2}\norm{\bx(t)}_2^4 \ge 	\frac{1-\eta\lambda}{2\eta\lambda}(1-e^{-4t\eta\lambda(1-\eta\lambda)})\undersigma^2
-\frac{1}{2}(1-\eta\lambda)^2 M^2e^{4\eta\lambda}\sqrt{\frac{1}{\lambda\eta}\ln\frac{2T^2}{\delta}}
\end{align*}
This completes the proof.
\end{proof}


\begin{lemma}[Norm upper bound]\label{lem:norm_upper_bound}
Under \Cref{assump:grad_norm_max} and additionally assume $\eta\lambda\le 0.1$.
Let $T_0 = \lceil \frac{1}{\eta\lambda}\rceil$. Let $t^*$ be the earliest step $t$ in $\{0,\ldots, T_0-1\}$ that $\eta^{-2}\norm{\bx(t)}_2^4\ge \frac{e^8(1-\eta\lambda)^2 \undersigma^2}{4\eta\lambda}$ and we denote $t^*=T_0$ if this doesn't happen in $\{0,\ldots,T_0-1\}$. For the case $t^*=T_0$, we have $\eta^{-2}\norm{\bx(T_0)}_2^4\le \frac{(1-\eta\lambda)^2\undersigma^2}{4\eta\lambda}$. On $\cE_T$, for any $t\ge t^*$,
	\begin{align}\label{eq:upper_bound_norm}
\begin{aligned}
	\eta^{-2}\norm{\bx(t+1)}_2^4 
	\le 	 e^{-4\lambda\eta(t-t^*)}\max\left\{2M^2e^{\left\lvert \ln \frac{2e^4M^2}{\norm{\bx(0)}_2^4\eta^{-2}}\right\rvert}, e^4\frac{\undersigma^2}{\eta\lambda}. \right\}	
	+ \frac{\oversigma^2}{\eta\lambda}.
\end{aligned}
\end{align}
Thus,  there exists $T_1 = T_0 + \frac{1}{4\eta\lambda}\max\left\{\ln \frac{M^2\eta\lambda}{\oversigma^2} + \left\lvert \ln \frac{2e^4M^2}{\norm{\bx(0)}_2^4\eta^{-2}}\right\rvert, 4  \right\}$, such that $\forall t\ge T_1$, $\eta^{-2}\norm{\bx(t+1)}_2^4  \le \frac{2\oversigma^2}{\eta\lambda}$.

\end{lemma}

\begin{proof}[Proof of \Cref{lem:norm_upper_bound}]
 If $t^*< T_0$, it holds that conditioned on $\cE_T$, for any $t^*\le t< T_0$, 
\begin{align*}
\begin{aligned}
	\eta^{-2}\norm{\bx_{t}}_2^4 \ge (1-\eta\lambda)^{4(t-t^*)} 	\eta^{-2}\norm{\bx(t^*)}_2^4 
	\ge (1-\eta\lambda)^{4(T_0-1)} 	\eta^{-2}\norm{\bx(t^*)}_2^4 \ge \frac{(1-\eta\lambda)^2 \undersigma^2}{4\eta\lambda}
\end{aligned}
\end{align*}
Therefore, for any $t\ge t^*$, we have
\begin{align}
\begin{aligned}\label{eq:upper_bound_norm_1}
	&\eta^{-2}\norm{\bx(t+1)}_2^4 \\
= &	(1-\eta\lambda)^4 \eta^{-2}\norm{\bx(t)}_2^4 + 2(1-\lambda\eta)^2\noise{\overline\bx(t)} + \frac{\noisebkf{t}}{\norm{\bx(t)}_2^4\eta^{-2}}\\
= & (1-\eta\lambda)^{4(t+1-t^*)} 	\eta^{-2}\norm{\bx(t^*)}_2^4 + 
	\underbrace{2\sum_{\tau = t^*}^t (1-\eta\lambda)^{4(t-\tau)+2} \EE[\noisek{\tau}\mid \bx(\tau)]}_{\text{(A)}}\\
	+ & \underbrace{2\sum_{\tau = t^*}^t (1-\eta\lambda)^{4(t-\tau)+2} \left(\noisek{\tau} - \EE [\noisek{\tau}\mid \bx(\tau)]  \right)}_{\text{(B)}}\\
+ &\underbrace{\sum_{\tau = t^*}^t (1-\eta\lambda)^{4(t-\tau)} \frac{\noisebkf{\tau}}{\norm{\bx(\tau)}_2^4\eta^{-2}}}_{\text{(C)}}.
\end{aligned}
\end{align}

Below we will upper-bound the terms (A), (B) and (C) on $\cE_T$ respectively.
\begin{enumerate}
\item[(A).] By \Cref{lem:power_series}, we have 
\begin{align}
\begin{aligned}
\text{(A)}\le 	2\sum_{\tau = t^*}^t (1-\eta\lambda)^{4(t-\tau)+2} \oversigma^2
\le  \frac{(1-\eta\lambda)^2e^{4\eta\lambda}}{2\eta\lambda}\oversigma^2\le \frac{e^{0.2}}{2\eta\lambda}\oversigma^2,
\end{aligned}
\end{align}
where in the last step we used $\eta\lambda \le 0.1$ and $e^x(1-x)\le 1$ for any $0\le x\le 1$.

\item[(B).] By the definition of event $\cE_T$,  we have 
\begin{align}
\text{(B)} \le (1-\eta\lambda)^2 \frac{M^2}{2}e^{4\eta\lambda}\sqrt{\frac{1}{\lambda\eta}\ln\frac{2T^2}{\delta}}	
\le \frac{(1-\eta\lambda)^2}{6\eta\lambda}\undersigma^2
\end{align}

\item[(C).] Combining the above analysis and \Cref{lem:norm_lower_bound}, we know conditioned on $\cE_T$, for any $t\ge t^*$, it holds $\norm{\bx(t)}_2^4/\eta^2\ge\frac{(1-\eta\lambda)^2 \undersigma^2}{4\eta\lambda} $. 

Therefore, by \Cref{lem:power_series}, we have 
\begin{align}
\begin{aligned}
	\text{(C)} \le \frac{4\eta\lambda M^4}{\undersigma^2}\sum_{\tau = t^*}^t (1-\eta\lambda)^{4(t-\tau)-2} 
	\le \frac{e^{4\eta\lambda}M^4}{(1-\eta\lambda)^2\undersigma^2}
\end{aligned}	
\end{align}
Under \Cref{assump:grad_norm_max}, we can further upper bound $(C)$ by $\frac{\undersigma^2}{9\eta\lambda e^{4\eta\lambda}(1-\eta\lambda)^2}\le \frac{\undersigma^2}{9\times \frac{8}{9}\times \frac{7}{8}\eta\lambda} =\frac{\undersigma^2}{7\eta\lambda}  $, where we used the fact that $ \eta\lambda\le 0.1$.
\end{enumerate}

What is left to do is to upper bound $\eta^{-2}\norm{\bx(t^*)}_2^4$. We proceed by discussing the following three cases respectively:
\begin{itemize}
	\item $t^*=0$. Then $\eta^{-2}\norm{\bx(t^*)}_2^4 = \eta^{-2}\norm{\bx(0)}_2^4$.
	\item $1\le t^*\le T_0-1$. In this case, we have 
\[\eta^{-1}\norm{\bx_{t^*-1}}_2^2\ge (1-\eta\lambda)^{2(t^*-1)}\eta^{-1}\norm{\bx(0)}_2^2\ge e^{-4(T_0-1)\eta\lambda}\eta^{-1}\norm{\bx(0)}_2^2\ge e^{-4}{\norm{\bx(0)}_2^2\eta^{-1}}. \]
 
Thus it holds that 
	\begin{align*}
	\begin{aligned}
		\eta^{-1}\norm{\bx(t^*)}_2^2 
		=& (1-\eta\lambda)^2\eta^{-1}\norm{\bx_{t^*-1}}_2^2 + \frac{\noisebk{t^*-1}}{\norm{\bx_{t^*-1}}_2^2\eta^{-1}}\\
		\le & (1-\eta\lambda)^2\sqrt{\frac{e^8 (1-\eta\lambda)^2\undersigma^2}{4\eta\lambda}}+ e^4\frac{M^2}{\norm{\bx(0)}_2^2\eta^{-1}}\\
		\le & 2\max\{\sqrt{\frac{e^8 \undersigma^2}{4\eta\lambda}}, e^4\frac{M^2}{\norm{\bx(0)}_2^2\eta^{-1}}\}\\
	\end{aligned}
	\end{align*}

\item $t^*=T_0$. Then we have $\eta^{-2}\norm{\bx(t^*)}_2^4 \le  \frac{(1-\eta\lambda)^2\undersigma^2}{4\eta\lambda}$.
\end{itemize}
Taking maximum over three cases, we have 
\begin{align}\label{eq:bound_x(t)star_norm}
\eta^{-2}\norm{\bx(t^*)}_2^4\le \max\left\{2e^4M^2e^{\left\lvert \ln \frac{2e^4M^2}{\norm{\bx(0)}_2^4\eta^{-2}}\right\rvert}, e^8\frac{\undersigma^2}{\eta\lambda}. \right\}	
\end{align}

Plugging \eqref{eq:bound_x(t)star_norm} back into \eqref{eq:upper_bound_norm_1}, we got for any $t\ge t^*$
\begin{align}\label{eq:upper_bound_norm_2}
\begin{aligned}
	&\eta^{-2}\norm{\bx(t+1)}_2^4 \\
	= &(1-\eta\lambda)^{4\eta\lambda (t+1-t^*)}\eta^{-2}\norm{\bx(t^*)}_2^4 + (A) +(B)+ (C)\\
	\le & e^{-4\lambda\eta(t-t^*)}\max\left\{2M^2e^{\left\lvert \ln \frac{2e^4M^2}{\norm{\bx(0)}_2^4\eta^{-2}}\right\rvert}, e^4\frac{\undersigma^2}{\eta\lambda}. \right\}	
	+ \frac{\oversigma^2}{\eta\lambda}, 
\end{aligned}
\end{align}
where we used the fact that $(0.5e^{0.2}+\frac{1}{6}+ \frac{1}{7}\approx 0.9202<1)$ in the last step.

Therefore there exists $T_1 = T_0 + \frac{1}{4\eta\lambda}\max\left\{\ln \frac{M^2\eta\lambda}{\oversigma^2} + \left\lvert \ln \frac{2e^4M^2}{\norm{\bx(0)}_2^4\eta^{-2}}\right\rvert, 4  \right\}$, such that for all $t\ge T_1$, $\eta^{-2}\norm{\bx(t)}_2^4  \le \frac{2\oversigma^2}{\eta\lambda}$.

\end{proof}

\sgdmain*
\begin{proof}
By \Cref{lem:taylor_expansion}, we have
\begin{align*}
L(\bx(t+1)) - L(\bx_{t}) \le 	-\frac{\eta}{1-\eta\lambda}\frac{\inner{\nabla L(\overline \bx(t))}{\nabla L_{\gamma_t}( \overline \bx(t))}}{\norm{\bx(t)}_2^2} + \frac{\rho\eta^2{\noisebk{t}}}{2(1-\eta\lambda)^2\norm{\bx(t)}_2^4}
\end{align*}

Summing up for $t=T_1$ to $T-1$, we have 
\begin{align*}
\begin{aligned}
&\sum_{t=T_1}^{T-1}\eta	\norm{\nabla L(\overline \bx(t))}_2^2\norm{\bx(t)}_2^{-2}	 
= \sum_{t=T_1}^{T-1}\eta	\norm{\nabla L(\bx(t))}_2^2	 \\
\le &  (1-\eta\lambda)\left(L(\bx_{T_1})- L(\bx_{T}) \right)
+ \underbrace{\sum_{t=T_1}^{T-1}  \frac{\rho\eta^2\E[{\noisebk{t}}\mid \bx(t)]}{2(1-\eta\lambda)\norm{\bx(t)}_2^4}}_{\text{(A)}}\\
+ & \underbrace{\sum_{t=T_1}^{T-1} \frac{\eta\inner{\nabla L( \overline \bx(t))}{\nabla L(\overline \bx(t))-\nabla L_{\gamma_t}(\overline \bx(t))}}{{\norm{\bx(t)}_2^2}}}_{\text{(B)}} \\
+ & \underbrace{\sum_{t=T_1}^{T-1}  \frac{\rho\eta^2\left({\noisebk{t}}-\E[{\noisebk{t}}\mid \bx(t)]\right)}{2(1-\eta\lambda)\norm{\bx(t)}_2^4}}_{\text{(C)}}
\end{aligned}
\end{align*}

Below we will give high-probability bounds for $(A)$, $(B)$ and $(C)$ respectively. For convenience, we will use $A(t),B(t),C(t)$ to denote the $t$th term in $(A)$, $(B)$ and $(C)$. 

\begin{claim}\label{clm:descent_lemma_A}
	$\cE_T\Longrightarrow$ $\forall  T_1\le t\le T,\ A(t)\le 2\sqrt{2}\rho\eta\lambda\frac{\oversigma^2}{\undersigma^2}$
\end{claim}

\begin{claim}\label{clm:descent_lemma_B}
	\emph{(B)} $=\sum_{t=T_1}^{T-1}B(t)$ is $\subg((T-T_1)\frac{4\pi^2\lambda\eta \rho^2M^2}{\undersigma^2},\cE_T)$  
\end{claim}

\begin{claim}\label{clm:descent_lemma_C}
	\emph{(C)} $=\sum_{t=T_1}^{T-1}C(t)$ is $\subg((T-T_1)\frac{4\rho^2\lambda^2\eta^2 M^4}{\undersigma^4},\cE_T)$  
\end{claim}

Here  \Cref{clm:descent_lemma_A} follows from that $2(1-\eta\lambda)\ge \sqrt{2}$ and \Cref{lem:norm_lower_bound}. Note by the choice of $T_1$, we can upper and lower bound $\norm{\bx(t)}_2$ by \Cref{lem:norm_lower_bound,lem:norm_upper_bound}, that is $\frac{\undersigma^2}{4\eta\lambda}\le \eta^{-2}\norm{\bx(t)}_2^2 \le \frac{2\oversigma^2}{\eta\lambda}$. Thus \Cref{clm:descent_lemma_B,clm:descent_lemma_C} is a direct consequence of  \Cref{lem:hp_gaussian}.

Thus we conclude w.p. $1-5\delta$,
\begin{align*}
\begin{aligned}
\sqrt{\frac{\lambda\eta}{2\oversigma^2}} \frac{1}{T-T_1}\sum_{t=T_1}^{T-1} \norm{\nabla L(\overline{\bx}(t))}_2^2
\le &\frac{L(\bx(T_1))-\min_{\bx} L(\bx)}{T-T_1} + 2\sqrt{2}\rho\eta\lambda\frac{\oversigma^2}{\undersigma^2} \\
+  & \sqrt{\frac{8\lambda\eta\ln\frac{2}{\delta}}{T-T_1}} \frac{\pi\rho M}{\undersigma} 
+ \sqrt{\frac{8\ln\frac{2}{\delta}}{T-T_1}} \lambda\eta\frac{M^2\rho}{\undersigma^2},
\end{aligned}
\end{align*}

rearranging it and applying \Cref{lem:loss_upper_bound}, we get 
\begin{align*}
\begin{aligned}
\frac{1}{T-T_1}\sum_{t=T_1}^{T-1} \norm{\nabla L(\overline{\bx}(t))}_2^2
\le &\frac{\pi^2\rho\oversigma}{(T-T_1)\sqrt{2\eta\lambda}} + 4\sqrt{\eta\lambda}\frac{\rho\oversigma^3}{\undersigma^2} \\
+  & \sqrt{\frac{\ln\frac{2}{\delta}}{T-T_1}} \frac{4\pi\rho M \oversigma}{\undersigma} 
+ \sqrt{\frac{\ln\frac{2}{\delta}}{T-T_1}} 4\sqrt{\lambda\eta}\frac{M^2\rho\oversigma}{\undersigma^2}.
\end{aligned}
\end{align*}

By \Cref{assump:grad_norm_max}, we have $\frac{\undersigma^2}{M^2} \ge 3\sqrt{{\lambda\eta}\ln\frac{2}{\delta}}$, and thus we have
\begin{align*}
\begin{aligned}
\frac{1}{T-T_1}\sum_{t=T_1}^{T-1} \norm{\nabla L(\overline{\bx}(t))}_2^2
\le &\frac{\pi^2\rho\oversigma}{(T-T_1)\sqrt{2\eta\lambda}} + 4\sqrt{\eta\lambda}\frac{\rho\oversigma^3}{\undersigma^2} 
+   \frac{4}{3}\sqrt{\frac{1}{(T-T_1)\eta\lambda}} \pi\rho\undersigma
+ \sqrt{\frac{1}{T-T_1}} \frac{4\rho\oversigma}{3}.
\end{aligned}
\end{align*}
This completes the proof.
\end{proof}

\section{Omitted Proofs for Convergence of SGD with Relative Global Clipping}\label{sec:proof_clipping}

\paragraph{Norm dynamics of clipped \sgd:} 
\begin{align}
\norm{\bx(t+1)}_2^2 = 	(1-\eta\lambda)^2 \norm{\bx(t)}_2^2 + \eta^2\min\left\{\frac{\noise{\overline\bx(t)}}{\norm{\bx(t)}_2^2}, \frac{2\lambda C}{\eta} \norm{\bx(t)}_2^2 \right\}.
\end{align}

\begin{lemma}[General Properties of $G_{P,C}$]\label{lem:property_of_f} For any $C>1$ and measure $P$ supported on $\RR^{\ge 0}$, it holds that
\begin{enumerate}
	\item $G_{P,C}$ is continuous and concave;
	\item $\sup_{\mu\ge 0}G_{P,C}(\mu) = G_{P,C}(\frac{1}{C}M_{P,\frac{1}{C}})$;
	\item $ \frac{1}{C}M_{P,\frac{1}{C}}\le \mu_{P,C} \le \mu_{P}$, where $\mu_P$ is the expectation of $P$.
\end{enumerate}
	
\end{lemma}

\begin{proof}[Proof of \Cref{lem:property_of_f}]
(1).  Note $\min\{x,\cdot\}$ is a continuous and concave function for any $x$, we know $G_{P,C}$ is a concave function. 
(2). When $G_{P,C}$ is differentiable, we have $G_{P,C}'(\mu) = C F_{P,C}'(C\mu) -1$. Let $G_{P,C}'(\mu)=0$ implies that $F_{P,C}'(C\mu) = \frac{1}{C}$. Note $F_{P,C}'(C\mu)=\PP_{t\sim P}[t> F_{P,C}]$, we know $G_{P,C}'(\frac{1}{C}M_{P,\frac{1}{C}})=0$. By concavity,  $\sup_{\mu\ge 0}G_{P,C}(\mu) = G_{P,C}(\frac{1}{C}M_{P,\frac{1}{C}})$. This argument can be easily generalized to non-differentiable case by using $G_{P,C}(\mu)$ must be larger than $G_{P,C}(\mu\pm \delta)$ for infinitesimal $\delta$.
(3). First note that $F_{P, C}(  M_{P, \frac{1}{C}})=\EE_{t\sim P}[\min\{t, M_{P, \frac{1}{C}}\}] \ge M_{P, \frac{1}{C}}\cdot \PP_{t\sim P}[t\ge M_{P, \frac{1}{C}}] = \frac{1}{C}M_{P, \frac{1}{C}} $. In other words, $G_{P,C}(\frac{1}{C} M_{P, \frac{1}{C}})\ge 0$. 

Now suppose $\frac{1}{C} M_{P, \frac{1}{C}}> \mu_{P, C}$. If $G_{P,C}(\frac{1}{C} M_{P, \frac{1}{C}})=0$, then by definition, $\frac{1}{C} M_{P, \frac{1}{C}}\le \mu_{P, C}$. If $G_{P,C}(\frac{1}{C} M_{P, \frac{1}{C}})>0$, by concavity, $G_{P,C}(\mu_{P, C})>0$, contradiction!
\end{proof}
\begin{theorem}\label{thm:unique_clipped_mean}[Classifications of solutions of $F_{P, C}(C\mu) = \mu$]
\ 
\begin{enumerate}
\item  If $\PP[x=0]<1-\frac{1}{C}$, then $F_{P, C}(C\mu) = \mu$ has exact two solutions which are $0$ and $\mu_{P,C}>0$;
\item  If $\PP[x=0]=1-\frac{1}{C}$, then $F_{P, C}(C\mu) = \mu$ for all $0\le \mu\le \frac{1}{C}M_{P,C}$ and $\mu_{P,C} = \frac{1}{C}M_{P,C}$;
\item If $\PP[x=0]>1-\frac{1}{C}$, then $F_{P, C}(C\mu) = \mu$ has only one solution which is $\mu_{P,C}=0$.
\end{enumerate}
\end{theorem}
\begin{proof}
 Suppose there are two solutions $0<\mu_1<\mu_2$. By concavity, we have $\forall 0\le \mu\le \mu_2$, $G_{P,C}(\mu)=0$. Thus $0=G_{P,C}(0)+G_{P,C}(\mu_2)=2g(\frac{\mu_2}{2})$, which implies that 
\begin{align*}
    \EE_{t\sim P}[\min\{t,C\mu_2\}] = 2\EE_{t\sim P}[\min\{t,\frac{C\mu_2}{2}\}] = \EE_{t\sim P}[\min\{2t,C\mu_2\}],
\end{align*}
that is, $\PP_{t\sim P}[t\ge C\mu_2 \vee t=0]=1$. Thus for any $0\le \mu\le \mu_2$, we have $G_{P,C}(\mu) =C\mu\PP[x\ge C\mu_2] -\mu =0$, which implies $\mu_2= \frac{1}{C} M_{P, \frac{1}{C}}$ and $\PP[x=0]=1-\frac{1}{C}$! 
\end{proof}

\begin{lemma}\label{lem:g_max}
Under \Cref{assump:g_max}, it holds that $G_{P,C_{\bx}}(\frac{1}{C}M_{P_{\bx},\frac{1}{C}})\ge \alpha_C \mu_{P_{\bx},C}$ for all $\bx\neq 0$.
\end{lemma}

\begin{proof}[Proof of \Cref{lem:g_max}]
	By definition, 
	\begin{align}
		G_{P,C_{\bx}}(\frac{1}{C}M_{P_{\bx},\frac{1}{C}}) = \EE_{t\sim P_{\bx}}[t\ind[t<M_{P_{\bx},C}]] + (\PP_{t\sim P_{\bx}}[t\ge M_{P_{\bx},C}]-\frac{1}{C})\cdot M_{P_{\bx},C}. 
	\end{align}
By the definition of the $\frac{1}{C}$-median, the second term is non-negative. The proof is completed by applying \Cref{assump:g_max}.
\end{proof}

\begin{lemma}[Lower and upped bounds for $G_{P_{\bx},C}$]\label{lem:lower_bound_G}
 Under \Cref{assump:g_max}, it holds that
	\begin{enumerate}
		\item $G_{P_{\bx},C}(\mu)\ge \alpha_C\mu$, for  $0\le \mu \le \frac{\mu_{P_{\bx},C}}{2}$;
		\item $G_{P_{\bx},C}(\mu)\ge \alpha_C(\mu_{P_{\bx},C}-\mu)$, for  $\frac{\mu_{P_{\bx},C}}{2}\le \mu \le \mu_{P_{\bx},C}$;
		\item $G_{P_{\bx},C}(\mu) \le -\alpha_C (\mu - \mu_{P_{\bx},C})$, for $\mu\ge \mu_{P_{\bx},C}$.
	\end{enumerate}
\end{lemma}
\begin{proof}[Proof of \Cref{lem:lower_bound_G}]
By \Cref{lem:g_max}, \Cref{assump:g_max} implies that $G_{P,C_{\bx}}(\frac{1}{C}M_{P_{\bx},\frac{1}{C}})\ge \alpha_C \mu_{P_{\bx},C}$ for all $\bx\neq 0$.
Further note that $G_{P,C_{\bx}}(0) = G_{P,C_{\bx}}(\mu_{P_{\bx}},C)=0$.
The claims (a), (b) and (c) are immediate by concavity of $G_{P,C_{\bx}}$.
\end{proof}
The above inequalities also directly imply the following version using $\mumin$ and $\mumax$ as thresholds.
\begin{lemma}[Uniform Lower and upped bounds for $G_{P_{\bx},C}$]\label{lem:uniform_lower_bound_G}
Under \Cref{assump:g_max}, it holds that for $\norm{\bx}_2=1$,
	\begin{enumerate}
		\item $G_{P_{\bx},C}(\mu)\ge \alpha_C\mu$, for  $0\le \mu \le \frac{\mumin}{2}$;
		\item $G_{P_{\bx},C}(\mu)\ge \alpha_C(\mumin-\mu)$, for  $\frac{\mumin}{2}\le \mu \le \mumin$;
		\item $G_{P_{\bx},C}(\mu) \le -\alpha_C (\mu - \mumax)$, for $\mu\ge \mumax$.
		\item $G_{P_{\bx},C}(\mu)\ge \frac{\alpha_C \mu}{4}$, for  $0\le \mu \le \frac{4\mumin}{5}$; (4. follows from Property 1. and 2.)
	\end{enumerate}
\end{lemma}

For convenience, we define $R_t:=\frac{2\lambda}{\eta}\norm{\bx(t)}_2^2$, $g_t:=\norm{\nabla L_{\gamma_t}( \bx(t))}_2^2$,  $\hat{g}_t:=\min\{CR_t, g_t\}$, $\tilde{g}_t:=R_t\hat{g}_t = \min\{CR_t^2, \norm{\nabla L_{\gamma_t}( \overline\bx(t))}_2^2\}$ and $\overline{g}_t:=\frac{\hat{g}_t}{R_t} = \min\{C, \frac{\norm{\nabla L_{\gamma_t}( \overline\bx(t))}_2^2}{R_t}\}$. Thus we have $\EE[\hat{g}_t\mid \bx(t)] = \mu_{P_{\bx(t)},C}$. We further define $\beta_l:=1-2\lambda^2\eta^2 + \eta^4\lambda^4 -4\eta\lambda \alpha_C(1-\eta\lambda)^2 =1- 4\eta\lambda\alpha_C +O(\eta^2\lambda^2)$ and $\beta_u:=1-2\lambda^2\eta^2 + \eta^4\lambda^4 -4\eta\lambda \alpha_C(1-\eta\lambda)^2 + 4C^2\eta^2\lambda^2 =1 - 4\eta\lambda\alpha_C + O(\eta^2\lambda^2)$.

Given an integer $T\ge 0$, let $\mathcal{E}^1_T$ be the event that  $\forall 0 \le t'\le  t \le T,$
\[\left| \sum_{s=t'}^t {\beta_l}^{t-s}\left( \tilde{g}_s - \EE[\tilde{g}_s\mid \bx(s)]\right) \ind\left[R^2_s \le \mumin\right] \right| \le \sqrt{C}\mumin\sqrt{\frac{1}{1-{\beta_l}^2}\ln\frac{2T^2}{\delta}}.\]

Let $\mathcal{E}^2_T$ be the event that  $\forall 0 \le t'\le  t \le T,$
\[\left| \sum_{s=t'}^t {\beta_l}^{t-s}\left( \tilde{g}_s - \EE[\tilde{g}_s\mid \bx(s)]\right) \ind\left[R^2_s \le 2\mumax\right] \right| \le 2\sqrt{C}\mumax\sqrt{\frac{1}{1-{\beta_l}^2}\ln\frac{2T^2}{\delta}}.\]

Let $\mathcal{E}^3_T$ be the event that  $\forall 0 \le t'\le  t \le T,$
\[\left| \sum_{s=t'}^t \overline{g}_s - \EE[\overline{g}_s\mid \bx(s)] \right| \le C\sqrt{T\ln\frac{2T^2}{\delta}}.  \]

\begin{lemma}\label{lem:clip_concentration}
$\PP[\mathcal{E}_T^i]\ge 1-\delta$, for $i=1,2,3$.
\end{lemma}
\begin{proof}[Proof of \Cref{lem:clip_concentration}]
Note the sequence in $\mathcal{E}_T^i$ are martingales whose differences are uniformly bounded ($\mumin, \mumax$ and $C$). The lemma follows directly from Hoeffding Inequality and Azuma Inequality.  
\end{proof}
\begin{theorem}[Norm lower bound with clipping: Warm Start]\label{lem:clip_norm_lower_bound}
	Suppose \Cref{assump:g_max} holds, with probability at least $1-\delta$ (or whenever $\mathcal{E}^1_T$ holds), if $R_t^2\ge \frac{3}{4}\mumin$, then for any $t'\ge t$, we have 
	\begin{align}
R_{t'}^2\ge \left(1- \frac{{\beta_l}^{t'-t}}{4}-O(\sqrt{\eta\lambda}) - \sqrt{\frac{2C}{\alpha_C}\eta\lambda\ln\frac{T^2}{\delta}}(1+O(\eta\lambda)) \right)\mumin	
\end{align}
\end{theorem}

\begin{proof}	
We first claim for any $t\le t'\le T$, conditioned on $\mathcal{E}^1_T$, it holds that $R_{t'}^2\ge \frac{\mumin}{2}$. Below we prove by contradiction. If not, let $t'$ be the smallest step such that $R_{t'}^2< \frac{\mumin}{2}$.
We let $t^*$ be the largest step between $t$ and $t'$ such that $R_{t^*}^2 \ge \mumin$ ($t^*=t-1$ is no such $t^*$ exists) Thus if $t^*\ge t$ then $R_{t^*+1}^2$ is at least $(1-\eta\lambda)^4R_t^2 = (1-O(\eta\lambda))\mumin$. Otherwise $t^*=t$ and it implies that $R_{t^*+1}^2 = R_{t}^2 = (\frac{3}{4}-O(\sqrt{\eta\lambda}))\mumin$. By the definition, we know for any $t^*+1\le s\le t'$, $R_s^2\le \mumin$.

Similar to  \Cref{eq:norm_dynamics_2}, we have 
\begin{align}\label{clip_norm_dynamic_lower}
\begin{aligned}
	R_{s+1}^2
	= & R_s^2(1-\eta\lambda)^4 + 4\eta\lambda(1-\eta\lambda)^2\tilde{g}_s + 4\eta^2\lambda^2 \tilde{g}_t^2\\
	\ge & R_s^2((1-\eta\lambda)^4+4\eta\lambda(1-\eta\lambda)^2 + 4C^2\eta^2\lambda^2)  \\
	+ &4\eta\lambda(1-\eta\lambda)^2 (\EE[\tilde{g}_s\mid \bx(s)]-R_s^2) + 4\eta\lambda(1-\eta\lambda)^2 (\tilde{g}_s - \EE[\tilde{g}_s\mid \bx(s)])
	\end{aligned}
\end{align}

Thus for any $s$ such that $\mumax\le R_s^2\le 2\mumax$, by \Cref{lem:uniform_lower_bound_G}, it holds that
\begin{align*}
G_{P_{\overline\bx(s)},C}(R_s^2) = \EE[\tilde{g}_s\mid \bx(s)]-R_s^2 \le \alpha_C(\mumin - R_s^2).	
\end{align*}

Thus, we have that
\begin{align*}
	R_{s+1}^2
	\ge & R_s^2(1-2\eta^2\lambda^2+\eta^4\lambda^4) \\
	+ &4\eta\lambda\alpha_C(1-\eta\lambda)^2 (\mumin - R_s^2) + 4\eta\lambda(1-\eta\lambda)^2 (\tilde{g}_s - \EE[\tilde{g}_s\mid \bx(s)])\\
	= & \beta_l R_s^2 + 4\eta\lambda\alpha_C(1-\eta\lambda)^2\mumin + 4\eta\lambda(1-\eta\lambda)^2 (\tilde{g}_s - \EE[\tilde{g}_s\mid \bx(s)]).
\end{align*}

That is, 
\begin{align*}
	&R_{s+1}^2 -\frac{4\eta\lambda\alpha_C(1-\eta\lambda)^2\mumin}{1-\beta_l}  \\
	\ge &   \beta_l (R_s^2 - \frac{4\eta\lambda\alpha_C(1-\eta\lambda)^2\mumin}{1-\beta_l}) \\
	 + &4\eta\lambda(1-\eta\lambda)^2 (\tilde{g}_s - \EE[\tilde{g}_s\mid \bx(s)])
\end{align*}

Applying the above inequality for $s=t^*+1,\ldots,t'-1$, we have that
\begin{align*}
	R_{t'}^2 \ge &\underbrace{{\beta_l}^{t'-t^*-1} \left(R_{t^*+1}^2-\frac{4\eta\lambda\alpha_C(1-\eta\lambda)^2\mumin}{1-\beta_l}\right)}_{(A)} \\
				+ & \underbrace{\frac{4\eta\lambda\alpha_C(1-\eta\lambda)^2\mumin}{1-\beta_l}}_{(B)} \\
				+ &\underbrace{ 4\eta\lambda(1-\eta\lambda)^2\sum_{s=t^*+1}^{t'}  {\beta_l}^{t-s}\left( \tilde{g}_s - \EE[\tilde{g}_s\mid \bx(s)]\right) \ind\left[R^2_s \le \mumin\right]}_{(C)}.
\end{align*}

For term (B), we have $1-\beta_u = 4\eta\lambda\alpha_C(1-\eta\lambda)^2(1+O(\eta\lambda))$ and thus $(B) = \mumin(1+O(\eta\lambda))$. Since $R_{t^*+1}\ge \frac{3}{4}\mumin$, it holds that $(A)\ge -{\beta_l}^{t'-t^*-1}(\frac{1}{4}+O(\sqrt{\lambda\eta}))\mumin \ge -(\frac{1}{4}+O(\sqrt{\lambda\eta}))\mumin$. Since $\mathcal{E}^1_T$ holds, we have 
\begin{align*}
|(C)| \le 4\eta\lambda(1-\eta\lambda)^2	\cdot \sqrt{C}\mumin\sqrt{\frac{1}{1-{\beta_l}^2}\ln\frac{2T^2}{\delta}}
 = \mumin\sqrt{\frac{2C}{\alpha_C}\eta\lambda\ln\frac{T^2}{\delta}}(1+O(\eta\lambda)) 
\end{align*}

Thus there's some constant $\iota$, such for $\eta\lambda \le \min\{\iota, \frac{\alpha_C}{64C\ln T^2/\delta}\}$, $(A) + (B) +(C)\ge (\frac{6-\sqrt{2}}{8}-O(\sqrt{\eta\lambda}))\mumin\ge \frac{\mumin}{2}$. This leads to a contradiction to the definition of $t'$. Thus for any $t\le t'\le T$, conditioned on $\mathcal{E}^1_T$, it holds that $R_{t'}^2\ge \frac{\mumin}{2}$. Furthermore, if $t^*\neq t$, then $R_{t^*+1}\ge (1-O(\sqrt{\eta\lambda}))\mumin$. Thus $(A)\ge -O(\sqrt{\eta\lambda})\mumin$. Otherwise if $t^* = t$, then $(A)\ge -{\beta_l}^{t'-t}(\frac{1}{4}+O(\sqrt{\lambda\eta}))\mumin$. Combine the bounds in these two cases, we conclude that 
\begin{align*}
R_{t'}^2\ge \left(1- \frac{{\beta_l}^{t'-t}}{4}-O(\sqrt{\eta\lambda}) - \sqrt{\frac{2C}{\alpha_C}\eta\lambda\ln\frac{T^2}{\delta}}(1+O(\eta\lambda)) \right)\mumin	
\end{align*}

\end{proof}

\begin{theorem}[Norm upper bound with clipping: Warm Start]\label{lem:clip_norm_upper_bound}
	Suppose \Cref{assump:g_max} holds, with probability at least $1-\delta$ (or whenever $\mathcal{E}^2_T$ holds), if $R_t^2\le \frac{3}{2}\mumax$, then for any $t'\ge t$, we have 
	\begin{align*}
R_{t'}^2\le \left(1 +  \frac{{\beta_l}^{t'-t}}{2}+O(\sqrt{\eta\lambda}) + \sqrt{\frac{2C}{\alpha_C}\eta\lambda\ln\frac{T^2}{\delta}}(1+O(\eta\lambda)) \right)\mumax	
\end{align*}
\end{theorem}

\begin{proof}[Proof of \Cref{lem:clip_norm_upper_bound}]	
We first claim for any $t\le t'\le T$, conditioned on $\mathcal{E}^2_T$, it holds that $R_{t'}^2\le 2\mumax$. Below we prove by contradiction. If not, let $t'$ be the largest step such that $R_{t'}^2> 2\mumax$.
We let $t^*$ be the largest step between $t$ and $t'$ such that $R_{t^*}^2 \le \mumax$ ($t^*=t-1$ is no such $t^*$ exists) Thus if $t^*\ge t$ then $R_{t^*+1}^2$ is at most $(1+2C\eta\lambda)^2R_t^2 = (1+2C\eta\lambda)^2\mumax$. Otherwise $t^*=t$ and it implies that $R_{t^*+1}^2 = R_{t}^2 \le \frac{3}{2}\mumax$. By the definition, we know for any $t^*+1\le s\le t'$, $R_s^2\ge \mumax$.

Similar to  \Cref{eq:norm_dynamics_2}, we have 
\begin{align}\label{eq:clip_norm_dynamic_upper}
\begin{aligned}
	R_{s+1}^2
	\le & R_s^2(1-\eta\lambda)^4 + 4\eta\lambda(1-\eta\lambda)^2\tilde{g}_s + 4 \eta^2\lambda^2\hat{g}_s^2\\
	\le & R_s^2((1-\eta\lambda)^4+4\eta\lambda(1-\eta\lambda)^2+4\eta^2\lambda^2C^2) \\
	+ &4\eta\lambda(1-\eta\lambda)^2 (\EE[\tilde{g}_s\mid \bx(s)]-R_s^2) + 4\eta\lambda(1-\eta\lambda)^2 (\tilde{g}_s - \EE[\tilde{g}_s\mid \bx(s)])
\end{aligned}
\end{align}

Thus for any $s$ such that $\mumax\le R_s^2$, by \Cref{lem:uniform_lower_bound_G}, it holds that
\begin{align*}
G_{P_{\overline\bx(s)},C}(R_s^2) = \EE[\tilde{g}_s\mid \bx(s)]-R_s^2 \ge \alpha_C(\mumax - R_s^2).	
\end{align*}

Thus, we have that 
\begin{align*}
	R_{s+1}^2
	\le & R_s^2(1-2\eta^2\lambda^2+\eta^4\lambda^4+4\eta^2\lambda^2C^2) \\
	+ &4\eta\lambda\alpha_C(1-\eta\lambda)^2 (\mumax - R_s^2) + 4\eta\lambda(1-\eta\lambda)^2 (\tilde{g}_s - \EE[\tilde{g}_s\mid \bx(s)])\\
	= & \beta_u R_s^2 + 4\eta\lambda\alpha_C(1-\eta\lambda)^2\mumax + 4\eta\lambda(1-\eta\lambda)^2 (\tilde{g}_s - \EE[\tilde{g}_s\mid \bx(s)]).
\end{align*}

That is, 
\begin{align*}
	&R_{s+1}^2 -\frac{4\eta\lambda\alpha_C(1-\eta\lambda)^2\mumax}{1-\beta_u}  \\
	\le &   \beta_u (R_s^2 - \frac{4\eta\lambda\alpha_C(1-\eta\lambda)^2\mumax}{1-\beta_u}) 
	 + 4\eta\lambda(1-\eta\lambda)^2 (\tilde{g}_s - \EE[\tilde{g}_s\mid \bx(s)])
\end{align*}

Applying the above inequality for $s=t^*+1,\ldots,t'-1$, we have 
\begin{align*}
	R_{t'}^2 \le &\underbrace{{\beta_u}^{t'-t^*-1} \left(R_{t^*+1}^2-\frac{4\eta\lambda\alpha_C(1-\eta\lambda)^2\mumax}{1-\beta_u}\right)}_{(A)} \\
				+ & \underbrace{\frac{4\eta\lambda\alpha_C(1-\eta\lambda)^2\mumax}{1-\beta_u}}_{(B)} \\
				+ &\underbrace{ 4\eta\lambda(1-\eta\lambda)^2\sum_{s=t^*+1}^{t'}  {\beta_u}^{t-s}\left( \tilde{g}_s - \EE[\tilde{g}_s\mid \bx(s)]\right) \ind\left[R^2_s \le 2\mumax\right]}_{(C)}.
\end{align*}

For term (B), we have $1-\beta_u = 4\eta\lambda\alpha_C(1-\eta\lambda)^2(1+O(\eta\lambda))$ and thus $(B) = \mumax(1+O(\eta\lambda))$. Since $R_{t^*+1}\le \frac{3}{2}\mumax$, it holds that $(A)\le {\beta_u}^{t'-t^*-1}(\frac{1}{2}+O(\sqrt{\lambda\eta}))\mumax \le (\frac{1}{2}+O(\sqrt{\lambda\eta}))\mumax$. Since $\mathcal{E}^2_T$ holds, we have that
\begin{align*}
|(C)| \le 8\eta\lambda(1-\eta\lambda)^2	\cdot \sqrt{C}\mumax\sqrt{\frac{1}{1-{\beta_u}^2}\ln\frac{2T^2}{\delta}}
 = 2\mumax\sqrt{\frac{2C}{\alpha_C}\eta\lambda\ln\frac{T^2}{\delta}}(1+O(\eta\lambda)) 
\end{align*}

Thus there's some constant $\iota$, such for $\eta\lambda \le \min\{\iota, \frac{\alpha_C}{64C\ln T^2/\delta}\}$, $(A) + (B) +(C)\le (\frac{6+\sqrt{2}}{4}+O(\sqrt{\eta\lambda}))\mumax\le 2\mumax$. This leads to a contradiction to the definition of $t'$. Thus for any $t\le t'\le T$, conditioned on $\mathcal{E}^1_T$, it holds that $R_{t'}^2\ge 2\mumax$. Furthermore, if $t^*\neq t$, then $R_{t^*+1}\le (1+O(\sqrt{\eta\lambda}))\mumax$. Thus $(A)\le O(\sqrt{\eta\lambda})\mumax$. Otherwise if $t^* = t$, then $(A)\le {\beta_u}^{t'-t}(\frac{1}{2}+O(\sqrt{\lambda\eta}))\mumax$. Combine the bounds in these two cases, we conclude that 
	\begin{align*}
R_{t'}^2\le \left(1 +  \frac{{\beta_l}^{t'-t}}{2}+O(\sqrt{\eta\lambda}) + \sqrt{\frac{2C}{\alpha_C}\eta\lambda\ln\frac{T^2}{\delta}}(1+O(\eta\lambda)) \right)\mumax	
\end{align*}

\end{proof}

\begin{theorem}[Norm Convergence of clipped \sgd]\label{thm:clip_norm_convergence}
	Suppose \Cref{assump:g_max} holds, for $\eta\lambda=O(\min\{1, \frac{\alpha_C}{C\ln T/\delta^2}\})$, with probability $1-3\delta$ (when $\cE_T^1$,$\cE_T^2$ and $\cE_T^3$ happens), there is a $T'= \frac{\max\left\{\ln \frac{R^2_0}{\mumax}, \ln \frac{ \mumin}{R_0^2}  \right\} +O(1)}{\alpha_C\eta\lambda}$, such that for all $T'\le t\le T$, we have 
	\begin{align*}
			\frac{\mumin}{2} \le R_t^2 \le 2\mumax .
	\end{align*}
	More concretely, we have that
	\begin{align*}
	    	R_t^2 \in [(1-\beta_l^{t-T'}) \mumin - \tilde{O}(\sqrt{\lambda\eta}),\quad  \mumax(1+\beta_u^{t-T'}) + \tilde{O}(\sqrt{\lambda\eta})].
	\end{align*}
\end{theorem}

\begin{proof}[Proof of \Cref{thm:clip_norm_convergence}]
We will prove the desired inequality always holds when $\cE_T^i$ holds, for $i=1,2,3$. We have already proved the result for the case where $\frac{3}{4}\mumin \le R_t^2\le \frac{3}{2}\mumax$ in \Cref{lem:clip_norm_lower_bound,lem:clip_norm_upper_bound}. Now we turn to the case where $R_0^2\ge \frac{3}{2}\mumax$ and $R_0^2\le \frac{1}{2}\mumin$. Our goal is to prove with high probability, that $R_t^2\in [\frac{3}{4}\mumin, \frac{3}{2}\mumax]$ for at least some $t<T'$.

Below we first show $\exists 0<t<T'$, $R_t^2\le \frac{3}{2}\mumax$. Otherwise, similar to \Cref{eq:clip_norm_dynamic_upper},
\begin{align}\label{eq:clip_norm_dynamic_upper_2}
\begin{aligned}
	R_{s+1}^2
	\le & R_s^2(1-\eta\lambda)^4 + 4\eta\lambda(1-\eta\lambda)^2\tilde{g}_s + 4 \eta^2\lambda^2\hat{g}_s^2\\
	\le & R_s^2((1-\eta\lambda)^4+4\eta\lambda(1-\eta\lambda)^2+4\eta^2\lambda^2C^2) \\
	+ &4\eta\lambda(1-\eta\lambda)^2 (\EE[\tilde{g}_s\mid \bx(s)]-R_s^2) + 4\eta\lambda(1-\eta\lambda)^2 (\tilde{g}_s - \EE[\tilde{g}_s\mid \bx(s)])
\end{aligned}
\end{align}

Thus for any $s$ such that $\frac{3}{2}\mumax\le  R_s^2$, by \Cref{lem:uniform_lower_bound_G}, it holds that
\begin{align*}
G_{P_{\overline\bx(s)},C}(R_s^2) = \EE[\tilde{g}_s\mid \bx(s)]-R_s^2 \ge \alpha_C(\mumax - R_s^2)\ge -\frac{\alpha_C}{3}R_s^2.	
\end{align*}

Thus, 
\begin{align*}
	R_{s+1}^2
	\le & R_s^2(1-2\eta^2\lambda^2+\eta^4\lambda^4+4\eta^2\lambda^2C^2) \\
	- &\frac{4}{3}\eta\lambda\alpha_C(1-\eta\lambda)^2  R_s^2 + 4\eta\lambda(1-\eta\lambda)^2 (\tilde{g}_s - \EE[\tilde{g}_s\mid \bx(s)])\\
	= &  R_s^2 \left( 1-2\eta^2\lambda^2+\eta^4\lambda^4+4\eta^2\lambda^2C^2- \frac{4}{3}\eta\lambda\alpha_C(1-\eta\lambda)^2 + 4\eta\lambda(1-\eta\lambda)^2 (\overline{g}_s - \EE[\overline{g}_s\mid \bx(s)]) \right)
\end{align*}

Note that $\overline{g}_s\le C$, we have
\begin{align*}
    \ln R_{s+1}^2-\ln R_s^2 \le -\frac{4}{3}\eta\lambda\alpha_C + \eta\lambda (\overline{g}_s - \EE[\overline{g}_s\mid \bx(s)]) + O(\eta^2\lambda^2)
\end{align*}

Since we assume $\forall 0\le t\le T'$, $R_t^2\ge \frac{3}{2}\mumax$, conditioned on $\mathcal{E}_T^3$, we have
\begin{align*}
    \ln \frac{3}{4} + \ln \mumax - \ln R_0^2\le  \ln R_{T'}^2 - \ln R_0^2 \le -\frac{4T}{3}\eta\lambda \alpha_C+ C\eta\lambda\sqrt{T\ln\frac{2T^2}{\delta}} +O(\eta^2\lambda^2T),
\end{align*}
which is in contradiction with the definition of $T'= \frac{\max\left\{\ln \frac{R^2_0}{\mumax}, \ln \frac{ \mumin}{R_0^2}  \right\} +O(1)}{\alpha_C\eta\lambda}$.

Now we show $\exists 0<t<T'$, $R_t^2\ge \frac{3}{4}\mumin$. Otherwise, similar to \Cref{eq:clip_norm_dynamic_upper},
\begin{align}\label{clip_norm_dynamic_lower_2}
\begin{aligned}
	R_{s+1}^2
	= & R_s^2(1-\eta\lambda)^4 + 4\eta\lambda(1-\eta\lambda)^2\tilde{g}_s + 4\eta^2\lambda^2 \tilde{g}_t^2\\
	\ge & R_s^2((1-\eta\lambda)^4+4\eta\lambda(1-\eta\lambda)^2 + 4C^2\eta^2\lambda^2)  \\
	+ &4\eta\lambda(1-\eta\lambda)^2 (\EE[\tilde{g}_s\mid \bx(s)]-R_s^2) + 4\eta\lambda(1-\eta\lambda)^2 (\tilde{g}_s - \EE[\tilde{g}_s\mid \bx(s)])
	\end{aligned}
\end{align}

Thus for any $s$ such that $R_s^2\le \frac{4}{5}\mumin$, by \Cref{lem:uniform_lower_bound_G}, it holds that
\begin{align*}
G_{P_{\overline\bx(s)},C}(R_s^2) = \EE[\tilde{g}_s\mid \bx(s)]-R_s^2 \ge \frac{\alpha_C}{4} R_s^2.
\end{align*}

Thus, we have that 
\begin{align*}
	R_{s+1}^2
	\ge & R_s^2(1-2\eta^2\lambda^2+\eta^4\lambda^4) \\
	+ &\eta\lambda\alpha_C(1-\eta\lambda)^2  R_s^2 + 4\eta\lambda(1-\eta\lambda)^2 (\tilde{g}_s - \EE[\tilde{g}_s\mid \bx(s)])\\
	= &  R_s^2 \left( 1-2\eta^2\lambda^2+\eta^4\lambda^4+\eta\lambda\alpha_C(1-\eta\lambda)^2 + 4\eta\lambda(1-\eta\lambda)^2 (\overline{g}_s - \EE[\overline{g}_s\mid \bx(s)]) \right)
\end{align*}

Note that $\overline{g}_s\le C$, we have that 
\begin{align*}
    \ln R_{s+1}^2-\ln R_s^2 \ge \eta\lambda\alpha_C + \eta\lambda (\overline{g}_s - \EE[\overline{g}_s\mid \bx(s)]) + O(\eta^2\lambda^2)
\end{align*}

Since we assume $\forall 0\le t\le T'$, $R_t^2\ge \frac{3}{2}\mumax$, conditioned on $\mathcal{E}_T^3$, we have
\begin{align*}
 \ln \mumax - \ln R_0^2\ge  \ln R_{T'}^2 - \ln R_0^2 \ge T\eta\lambda \alpha_C- C\eta\lambda\sqrt{T\ln\frac{2T^2}{\delta}} +O(\eta^2\lambda^2T),
\end{align*}
which is in contradiction with the definition of $T'= \frac{\max\left\{\ln \frac{R^2_0}{\mumax}, \ln \frac{ \mumin}{R_0^2}  \right\} +O(1)}{\alpha_C\eta\lambda}$.

\end{proof}
%

\begin{proof}[Proof of \Cref{thm:clip_sgd_main}]
	The proof of \Cref{alg:clipped_sgd} is almost identical to that of \Cref{thm:sgd_main}, except replacing $M$ by $2\mumax$, $\oversigma$ by $\mumax$, $\undersigma$ by $\mumin$ since the clipped stochastic gradient has smaller maximum norm, maximum covariance and smaller covariance.
\end{proof}

\section{Convergence of SGD for multi-group scale invariant functions}\label{sec:multi_group}

	In this section we extend our results to the multi-group scale invariant setting, which is quite common in practice, \emph{e.g.} a feedforward network with normalization after each layer. By \Cref{defi:multi_group_scale_invariance}, multi-group scale invariant function is also scale invariant. However, it violates the assumption that the smoothness and the expectation of stochastic gradient norm square is lower bounded on unit sphere (indeed the loss function is not defined at everywhere on unit sphere), and thus needs to be treated separately. A simple example would be $L(\bx,\by) = L(\frac{\bx}{\norm{\bx}_2}, \frac{\by}{\norm{\by}_2})$, the loss $L$ is undefined at any point where $\norm{x}_2 =1$ and $\by=\bm{0}$. Yet our analysis for single scale invariant parameter group can still extend to this case, with a similar assumption that the expected gradient norm square is lower bounded.
	
	Let $d_1,\ldots,d_K$ be positive integers with $d=\sum_{k=1}^Kd_k$. For $\bx\in \RR^d= \RR^{d_1}\times \ldots \times \RR^{d_K}$, we use  $s_k$ to denote $\sum_{i\le k}d_i$ and $\bx_k$ to denote the vector $[x_{s_{k-1}}, \ldots ,x_{s_k-1}]^\top$.  For convenience,  we define $\nabla_k f(\bx) = \frac{\partial f(\bx)}{\partial \bx_k}$ for any $1\le k\le K$.
\begin{definition}\label{defi:multi_group_scale_invariance} 
	Given $d_1,\ldots, d_K$ and a cone $U\subset \RR^d$, we say a function $f:U\to \RR$ is \emph{multi-group scale invariant}  iff $f(\bx_1,\ldots,\bx_K) = f(c_1\bx_1,\ldots,c_K\bx_K)$ for any $\bx\in U$ and $c_k>0$ for $1\le k\le K$.
\end{definition}

\paragraph{Setting:} Similarly, we assume there exists constants $\undersigma_k$ and $\oversigma_k$, such that $\undersigma_k^2 \le \E \norm{ \nabla_k L_\gamma (\bx)}_2^2\le \oversigma_k^2$, for any $\bx$ such that $\norm{\bx_k}_2=1$. In this subsection, we define $\rho:=\max\limits_{\norm{\bx_k}_2=1,\forall k}\lambda_{\max}(\nabla^2 L(\bx))$.

\begin{condition}\label{assump:grad_norm_max_multi}  $\frac{\undersigma_k^2}{M_k^2} \ge 3e^{4\eta\lambda}\sqrt{{\lambda\eta}\ln\frac{2T^2}{\delta}}$.
 	
\end{condition}

\begin{restatable}[\sgdwd, Multi-group Scale Invariance]{theorem}{sgdmultimain}\label{thm:sgd_main_multi}
With probability $1-(K+2)\delta$, it holds that
\begin{align}\label{eq:main_multi}
&\frac{\sqrt{{\lambda\eta/2}}}{\sum_{k=1}^K \undersigma_k } \frac{1}{T-T_1}\sum_{t=T_1}^{T-1} \norm{\nabla L(\overline{\bx}(t))}_2^2\nonumber\\
\le &\frac{\pi^2 \rho}{T-T_1} + 2\sqrt{2}\rho\eta\lambda \sum_{k=1}^K\frac{\oversigma_k^2}{\undersigma_k^2} \\
+  & \sqrt{\frac{8\lambda\eta\ln\frac{2}{\delta}}{T-T_1}}\pi\rho \sum_{k=1}^K \frac{M_k}{\undersigma_k} 
+ \sqrt{\frac{8\ln\frac{2}{\delta}}{T-T_1}} \lambda\eta\rho \sum_{k=1}^K\frac{M_k^2}{\undersigma_k^2},\nonumber
\end{align}
where $T_1 = \frac{1}{4\eta\lambda}\max_k\left\{\ln \frac{M_k^2\eta\lambda}{\oversigma_k^2} + \left\lvert \ln \frac{2e^4M_k^2}{\norm{\bx_k(0)}_2^4\eta^{-2}}\right\rvert, 8  \right\}$.	
\end{restatable}

Following the same strategy, we can prove the multi-group counterpart of norm convergence result, \Cref{lem:whp_ET}. Given a integer $T\ge 0$, let $\mathcal{E}_{T,k}$ be the event that  $\forall 0 \le t'\le  t \le T-1,$
\begin{align*}
\left| \sum_{\tau=t'}^t (1-\eta\lambda)^{4(t-\tau)}\left(\noisebkm{\tau}-\E[\noisebkm{\tau}\mid \overline\bx(\tau)]\right)\right| \le  e^{4\eta\lambda}\cdot\frac{M_k^2}{4}\sqrt{\frac{1}{\lambda\eta}\ln\frac{2T^2}{\delta}}.
\end{align*}

\begin{lemma}\label{lem:whp_ETk}
	For any $0\le t'\le t\le T-1$, $1\le k\le K$, it holds that
	\begin{align*}
	\sum_{\tau=t'}^t (1-\eta\lambda)^{4(t-\tau)}\left(\noisebkm{\tau}-\E[\noisebkm{\tau}\mid \bx(\tau)]\right) \sim \subg(\frac{e^{8\eta\lambda}M_k^4}{32} )	
	\end{align*}
Thus we have $\PP[\mathcal{E}_{T,k}]\ge 1-\delta$ by \Cref{lem:subgaussian}.
\end{lemma}

The following theorem is a restatement of \Cref{lem:norm_upper_bound,lem:norm_lower_bound} in the context of multi-group scale invariance. 

\begin{lemma}\label{lem:norm_bound_multi}
Under \Cref{assump:grad_norm_max_multi},  there exists $T_1 = \frac{1}{4\eta\lambda}\max_k\left\{\ln \frac{M_k^2\eta\lambda}{\oversigma_k^2} + \left\lvert \ln \frac{2e^4M_k^2}{\norm{\bx_k(0)}_2^4\eta^{-2}}\right\rvert, 8  \right\}$, such that $\forall t\ge T_1$, $\frac{\undersigma_k^2}{4\eta\lambda} \le \eta^{-2}\norm{\bx(t)}_2^4  \le \frac{2\oversigma_k^2}{\eta\lambda}$, conditioned on $\cup_{k=1}^K \mathcal{E}_{T,k}$.	
\end{lemma}


The proof of \Cref{thm:sgd_main_multi} is a natural generalization of \Cref{thm:sgd_main}.
\begin{proof}[Proof of \Cref{thm:sgd_main_multi}]
Setting $\bx = (1-\eta\lambda)\bx(t)$ in  \Cref{lem:taylor_expansion_multi}, we have
\begin{align*}
L(\bx(t+1)) - L(\bx_{t}) \le 	-\frac{\eta}{1-\eta\lambda}\inner{\nabla L( \bx(t))}{\nabla L_{\gamma_t}(  \bx(t))}
+ \sum_{k=1}^K\frac{\rho\eta^2{\noisebkm{t}}}{2(1-\eta\lambda)^2\norm{\bx_k(t)}_2^4}
\end{align*}

For convenience we define $\hat{\bx} = [\frac{\bx_1^\top}{\norm{\bx_1}_2}, \ldots, \frac{\bx_K^\top}{\norm{\bx_K}_2}]^\top$. Summing up for $t=T_1$ to $T-1$, we have 
\begin{align*}
\begin{aligned}
&\sum_{t=T_1}^{T-1}\eta	\norm{\nabla L(\overline \bx(t))}_2^2\norm{\bx(t)}_2^{-2}	 
= \sum_{t=T_1}^{T-1}\eta	\norm{\nabla L(\bx(t))}_2^2	 \\
\le &  (1-\eta\lambda)\left(L(\bx_{T_1})- L(\bx_{T}) \right)
+ \underbrace{\sum_{t=T_1}^{T-1}  \sum_{k=1}^K \frac{\rho\eta^2\E[{\noisebkm{t}}\mid \bx(t)]}{2(1-\eta\lambda)\norm{\bx_k(t)}_2^4}}_{\text{(A)}}\\
+ & \underbrace{\sum_{t=T_1}^{T-1} \sum_{k=1}^K\frac{\eta\inner{\nabla_k L( \hat \bx(t))}{\nabla_k L(\hat \bx(t))-\nabla_k L_{\gamma_t}(\hat \bx(t))}}{{\norm{\bx_k(t)}_2^2}}}_{\text{(B)}} \\
+ & \underbrace{\sum_{t=T_1}^{T-1}  \sum_{k=1}^K\frac{\rho\eta^2\left({\noisebkm{t}}-\E[{\noisebkm{t}}\mid \bx(t)]\right)}{2(1-\eta\lambda)\norm{\bx_k(t)}_2^4}}_{\text{(C)}}
\end{aligned}
\end{align*}

Below we will give high-probability bounds for $(A)$, $(B)$ and $(C)$ respectively. For convenience, we will use $A(t),B(t),C(t)$ to denote the $t$th term in $(A)$, $(B)$ and $(C)$. 

\begin{claim}\label{clm:descent_lemma_A_multi}
	$\cup_{k=1}^K\cE_{T,k}\Longrightarrow$ $\forall  T_1\le t\le T,\ A(t)\le 2\sqrt{2}\rho\eta\lambda\sum_{k=1}^K\frac{\oversigma_k^2}{\undersigma_k^2}$
\end{claim}

\begin{claim}\label{clm:descent_lemma_B_multi}
	\emph{(B)} $=\sum_{t=T_1}^{T-1}B(t)$ is $\subg(4\pi^2\lambda\eta\rho^2(T-T_1)\left(\sum_{k=1}^K\frac{ M_k}{\undersigma_k}\right)^2,\cup_{k=1}^K\cE_{T,k})$  
\end{claim}

\begin{claim}\label{clm:descent_lemma_C_multi}
	\emph{(C)} $=\sum_{t=T_1}^{T-1}C(t)$ is $\subg(4\rho^2\lambda^2\eta^2 (T-T_1)\left(\sum_{k=1}^K\frac{M_k^2}{\undersigma_k^2}\right)^2,\cup_{k=1}^K\cE_{T,k})$  
\end{claim}

Here  \Cref{clm:descent_lemma_A_multi} follows from that $2(1-\eta\lambda)\ge \sqrt{2}$ and \Cref{lem:norm_lower_bound}. Note by the choice of $T_1$, we can upper and lower bound $\norm{\bx(t)}_2$ by \Cref{lem:norm_bound_multi}, that is $\frac{\undersigma_k^2}{4\eta\lambda}\le \eta^{-2}\norm{\bx_k(t)}_2^2 \le \frac{2\oversigma_k^2}{\eta\lambda}$. Thus \Cref{clm:descent_lemma_B_multi,clm:descent_lemma_C_multi} is a direct consequence of  \Cref{lem:hp_gaussian}.

Thus by Chernoff bound~(\Cref{lem:subgaussian}), with probability at least $1- (K+2)\delta$, \Cref{eq:main_multi} holds.
\end{proof}

%

\end{document}